\documentclass[nohyperref]{article}

\usepackage{microtype}
\usepackage{graphicx}
\usepackage{subfigure}
\usepackage{booktabs} 

\usepackage{hyperref}

\usepackage[accepted]{icml2022}

\usepackage{liyangicml}
\usepackage{tcolorbox}
\usepackage{enumitem}
\usepackage{caption}

\usepackage{amsmath}
\usepackage{amssymb}
\usepackage{mathtools}
\usepackage{amsthm}

\newcommand{\relu}[1]{\paren{#1}_+}

\newcommand{\mcD}{\mathcal D}
\newcommand{\mcE}{\mathcal E}
\newcommand{\mcG}{\mathcal G}
\newcommand{\mcH}{\mathcal H}
\newcommand{\mcX}{\mathcal X}

\newcommand{\mcV}{\mathcal V}

\newcommand*\bfell{\ensuremath{\boldsymbol\ell}}

\newcommand{\G}{\mathcal{G}}
\newcommand{\error}{\operatorname{error}}

\newcommand{\proj}{\mathrm{proj}}

\newcommand{\BuildTD}{\textsc{TopDownDT}}

\DeclarePairedDelimiter\ceil{\lceil}{\rceil}
\DeclarePairedDelimiter\floor{\lfloor}{\rfloor}
\newcommand{\Tribes}{\textsc{Tribes}}

\newcommand{\paren}[1]{\left({#1}\right)}
\DeclareMathOperator*{\argmax}{arg\,max}

\usepackage[capitalize,noabbrev]{cleveref}

\theoremstyle{plain}
\newtheorem{theorem}{Theorem}[section]
\newtheorem{proposition}[theorem]{Proposition}
\newtheorem{lemma}[theorem]{Lemma}
\newtheorem{fact}[theorem]{Fact}
\newtheorem{corollary}[theorem]{Corollary}
\theoremstyle{definition}
\newtheorem{definition}[theorem]{Definition}

\theoremstyle{remark}
\newtheorem{remark}[theorem]{Remark}

\usepackage[textsize=tiny]{todonotes}

\icmltitlerunning{Popular decision tree algorithms are provably noise tolerant }

\begin{document}

\twocolumn[
\icmltitle{Popular decision tree algorithms are provably noise tolerant }



\icmlsetsymbol{equal}{*}

\begin{icmlauthorlist}
\icmlauthor{Guy Blanc}{equal,stan}
\icmlauthor{Jane Lange}{equal,mit}
\icmlauthor{Ali Malik}{equal,stan}
\icmlauthor{Li-Yang Tan}{equal,stan}
\end{icmlauthorlist}

\icmlaffiliation{stan}{Department of Computer Science, Stanford University}
\icmlaffiliation{mit}{Department of Computer Science, Massachusetts Institute of Technology}

\icmlcorrespondingauthor{Guy Blanc}{gblanc@stanford.edu}
\icmlcorrespondingauthor{Jane Lange}{jlange@mit.edu}
\icmlcorrespondingauthor{Ali Malik}{malikali@cs.stanford.edu}
\icmlcorrespondingauthor{Li-Yang Tan}{liyang@cs.stanford.edu}

\icmlkeywords{Machine Learning, Decision trees, Learning theory, Adversarial, Noise, Robustness, ICML}

\vskip 0.3in
]



\printAffiliationsAndNotice{\icmlEqualContribution} 

\begin{abstract} 

Using the framework of boosting, we prove that all impurity-based decision tree learning algorithms, including the classic ID3, C4.5, and CART, are highly noise tolerant. Our guarantees hold under the strongest noise model of nasty noise, and we provide near-matching upper and lower bounds on the allowable noise rate. We further show that these algorithms, which are simple and  have long been central to everyday machine learning, enjoy provable guarantees in the noisy setting that are unmatched by existing algorithms in the theoretical literature on decision tree learning.   Taken together, our results add to an ongoing line of research that seeks to place the empirical success of these practical decision tree algorithms on firm theoretical footing. 



\end{abstract} 


\section{Introduction} 
Decision trees have been central to machine learning since its early days.  They give a simple way to represent a dataset in a hierarchical and logical manner, and they are perhaps the most canonical example of an intepretable  model.   They are also quick to evaluate, with evaluation time scaling with the depth of the tree, a quantity that is typically exponentially smaller than the overall number of nodes.   Classic decision tree learning algorithms such as ID3~\cite{Qui86}, C4.5~\cite{Qui93}, and CART~\cite{BFSO84}, as well as tree-based ensemble methods such as random forests~\cite{Bre01} and XGBoost~\cite{CG16}, are  therefore standard techniques in the modern machine learning toolkit.

{\sl Impurity-based} algorithms are a broad class that captures ID3, C4.5, CART, and indeed essentially all decision tree algorithms used in practice.  These algorithms build a binary decision tree  for a labeled dataset $S$ in a greedy top-down fashion.  Each algorithm $\mathcal{A}_\G$ is defined by an impurity function $\G : [0,1]\to [0,1]$ and a class $\mathcal{H}$ of allowable splitting functions.  The root of the tree built by $\mathcal{A}_{\G}$ corresponds to the split of $S$ into $S_0$ and $S_1$ by a function $h \in \mathcal{H}$ that maximizes the {\sl purity gain with respect to $\G$}.  The left and right subtrees are built by recursing on $S_0$ and $S_1$ respectively.  We elaborate on this framework in the body of the paper, mentioning for now that  the standard algorithms ID3, CART, and C4.5 can all be cast within this framework: ID3 and C4.5 use the binary entropy function $\G(p) = \textnormal{H}_2(p)$ and the associated purity gain is commonly called information gain, whereas CART uses the Gini impurity function $\G(p) = 4p(1-p)$. 

\paragraph{Prior work on provable performance guarantees.} Motivated by the empirical success of impurity-based decision tree algorithms, a fruitful and ongoing line of  work has focused on establishing provable guarantees on their performance in a variety of models and settings~\cite{KM96,Kea96,DKM96,FP04,Lee09,BLT-ITCS,BLT-ICML,BDM19,BDM20,BLQT21}.

However, these existing results either only hold in the noiseless setting or require strong and stylized assumptions that limit their practical relevance.  For example, the recent work of \citet{BLT-ICML} provides guarantees under the assumptions that examples are distributed according to a product distribution, that noise only affects the labels and not the features (i.e.~agnostic noise~\citep{Hau92,KSS94}), and that the corrupted labels are monotone in the features.  


\subsection{Our contributions}  

Using the framework of boosting, we prove that all impurity-based decision tree algorithms are highly noise tolerant in a fully general setting: our results apply to arbitrary distributions over features and labels, and we consider the strongest noise model, allowing for adversarial corruptions of both features and labels (i.e.~the nasty noise model of \citet{BEK02}, also known as  strong contamination).  Equivalently, we give the first formal guarantees on the robustness of impurity-based decision tree algorithms to distributional shifts; we elaborate on this perspective in the body of the paper.  We further provide near-matching upper and lower bounds on the allowable noise rate, showing that our analysis is essentially optimal.

Lastly, specializing this result to the setting of product distributions over features, we show that these algorithms enjoy provable guarantees in the noisy setting that are unmatched by existing algorithms in the sizeable theoretical literature on decision tree learning.  Most of these algorithms were developed after the invention of ID3, C4.5, and CART in the 1980s, and are more complicated and much less used in practice.



In more detail, our first result is the following:

\begin{theorem}[Impurity-based algorithms are noise-tolerant boosting algorithms; see~\Cref{thm:main_formal} for formal version]
\label{thm:main}
\label{thm:main-intro} 
For all impurity functions $\G$ and distributions $\mathcal{D}$ over features and labels, w.h.p.~over the draw of a sample $\bS$ from  $\mathcal{D}$, if $\mathcal{A}_{\mathcal{G}}$ is trained on an $\eta$-nasty-noise corruption $\wh{\bS}$ of $\bS$ where $\eta \le O(\eps\gamma)$, as long as the internal nodes of the tree are  $\gamma$-advantage hypotheses, growing the tree to size $\exp(O(1/\gamma^2\eps^2))$ achieves error $\le \eps$.  
\end{theorem} 

In one of the first papers to study impurity-based decision tree algorithms from a theoretical perspective, \citet{KM96} showed that they can be viewed as boosting algorithms.  \Cref{thm:main-intro} generalizes these results of \citet{KM96} to the setting of adversarial noise and shows that these algorithms are in fact highly noise-tolerant boosting algorithms.  

\paragraph{Optimality of~\Cref{thm:main-intro}.} Next, we show that the quantitative parameters of~\Cref{thm:main-intro} are essentially optimal.  First, even in the noiseless setting it can be seen that there are target functions for which the tree needs to be grown to size $\exp(\Omega(1/\gamma^2))$ to achieve high accuracy; this was already observed in~\citet{KM96,freund}. Second, we prove that the guarantees of~\Cref{thm:main-intro} cannot hold for noise rates $\eta \ge \tilde{\Omega}(\eps\gamma)$, even under strong feature and distributional assumptions: 

\begin{theorem}[Near-matching lower bound on noise rate; see~\Cref{thm:lb} for formal version]
\label{thm:lb-intro} 
Let the feature space be $\mathcal{X} = \bits^d$ and $\eta \ge \tilde{\Omega}(\eps\gamma)$.  There is a distribution $\mathcal{D}$ whose marginal over $\mathcal{X}$ is uniform, such that for all impurity functions $\mathcal{G}$, w.h.p.~over the draw of a sample $\bS$ from $\mathcal{D}$, there is an $\eta$-nasty-noise corruption $\wh{\bS}$ of $\bS$ such that if $\mathcal{A}_{\G}$ is trained on $\wh{\bS}$, even if all the internal nodes of the tree are  $\gamma$-advantage hypotheses, the tree has to be grown to size $2^{\Omega(d)}$ in order to achieve error $\le \eps$. 
\end{theorem}

A weaker lower bound of $\eta \ge \Omega(\sqrt{\gamma})$ for constant $\eps$ follows from the work of \citet{DSFTWW15}. \Cref{thm:lb-intro} improves this to the near-optimal $\eta \ge \tilde{\Omega}(\eps\gamma)$. A key ingredient in our proof is the celebrated Kahn-Kalai-Linial theorem~\cite{KKL88} from discrete Fourier analysis.  We contrast the dimension-independent bound, $\exp(O(1/\gamma^2\eps^2))$, on the size of tree in~\Cref{thm:main-intro} with the $2^{\Omega(d)}$ lower bound in~\Cref{thm:lb-intro}. 
\paragraph{Improving on the theoretical state of the art.} Specializing~\Cref{thm:main-intro} to the setting of product distributions over binary features, we further show the following: 
\begin{theorem}[Learning monotone decision trees in the presence of nasty noise; see~\Cref{thm:monotone_dt} for formal version]
\label{thm:monotone-intro} 
For any product distribution over $\bits^d$, impurity-based decision tree algorithms learn size-$s$ monotone decision trees in the presence of nasty noise in $\poly(d)\cdot s^{O(\log s)}$ time.
\end{theorem} 

This problem of learning decision trees in the setting of product distributions has been intensively studied in learning theory~\cite{Han93,Bsh93,KM93,BFJKMR94,JS06,OS07,GKK08,Lee09,KST09,HKY18,CM19,BDM20,BLT-ITCS,BLQT21focs}.  Many real-world learning scenarios are naturally monotone in nature, and relatedly, monotonicity is a commonly studied  assumption in learning theory.  

Prior to our work, the only algorithms provably resilient to nasty noise run in time $d^{O(\log s)}$~\cite{LMN93,KKMS08,BLT-ICALP}, which is only $\poly(d)$ for constant $s$.  These algorithms do not resemble the impurity-based algorithms used in practice. \Cref{thm:monotone-intro} therefore gives the first $\poly(d)$-time algorithm for any $s = \omega_d(1)$; our running time remains $\poly(d)$ for $s$ as large as $2^{O(\sqrt{\log d})}$.    For the weaker model of agnostic noise, recent work of~\cite{BLT-ICML} gives a $\poly(d)\cdot s^{O(\log s)}$ time algorithm.

We discuss other related work in~\Cref{apx:other-related-work}. 

\section{Preliminaries}
%
\paragraph{Notation.} 
We use {\bf boldface} (e.g.~$\bx \sim \mcD$) to denote random variables.  We consider the binary classification setting where we have a distribution $\mcD_X$ over an arbitrary domain $\mcX$ and a (randomised) classification function $\mcD_{Y=1 | X} : \mcX \to [0,1]$. Together, these define a distribution $\mcD$ over $\mcX \times \{0, 1\}$.  The goal of a learning algorithm is to use i.i.d.~samples $(\bx,\by)\sim \mathcal{D}$ to construct a hypothesis $T : \mcX \to \{0, 1\}$ that achieves low error on $\mcD$, where error is defined as:

\begin{equation}
    \error_\mcD[T] \coloneqq \Prx_{(\bx, \by) \sim \mcD}[T(\bx) \neq \by].
\end{equation}
We also define $\mu(\mcD) \coloneqq \Pr_\mcD [\by = 1]$, the bias of $\mathcal{D}$ towards the $1$-label, and $\eps(\mcD) \coloneqq \min\{\mu(\mcD), 1 - \mu(\mcD)\}$ to be the error of predicting the majority label on $\mcD$.
%

\paragraph{Decision tree hypotheses.} 
 We consider binary decision trees $T : \mathcal{X}\to\zo$ whose internal nodes $v$ are labeled by functions $h_v : \mathcal{X}\to \zo$ from a class $\mathcal{H}$ of allowable splitting functions.  The most standard class of splitting functions is the set of thresholds of a single feature, $h(x) = \Ind[x_i \ge \theta]$, though our guarantees apply to any arbitrary class $\mathcal{H}$.   Each instance $x \in \mcX$ follows a unique root-to-leaf path in $T$: at any internal node $v$, it follows either the left or right branch depending on the result of $h_v(x)$, until a leaf $\ell$ is reached.   The set of leaves $\ell \in \mathrm{leaves}(T)$ therefore form a partition of $\mcX$. We write $w_\mcD(\ell)$ to denote the probability that $\bx\sim \mathcal{D}_X$ reaches $\ell$ and we write $\mathcal{D}_\ell$ to denote $\mathcal{D}$ conditioned on $\bx$ reaching $\ell$.  We write $\bell\sim (T,\mathcal{D})$ to denote 
the draw of a random leaf where each leaf $\ell \in \mathrm{leaves}(T)$ is given weight $w_\mathcal{D}(\ell)$.

The prediction at a leaf $\ell$ is then taken to be the majority label of points that reach that leaf, i.e.~$\Ind[\mu(\mathcal{D}_\ell)\ge \frac1{2}]$.   The error of $T$ on $\mcD$ is therefore given by the sum of the error at each leaf incurred by predicting the majority label, weighted by the probability of reaching that leaf:
\[ 
    \error_\mcD[T] = \sum_{\ell \in \mathrm{leaves}(T)} w_\mcD(\ell) \eps(\mcD_\ell) = \Ex_{\mathbf{\bfell} \sim (T,\mcD)} [\eps(\mcD_{\mathbf{\bfell}})].
\]


We will need the notion of {\sl feature influences} in the context of binary features: 

\begin{definition}[Influence]
\label{def:inf}
Let $f : \bits^d \to \bits$ be a function and $\mcD_X = \mcD_X^{(1)} \times \ldots \times \mcD_X^{(d)}$ be a production distribution over $\bits^d$. For $i\in [d]$, the {\sl influence of feature $i$ on $f$}, denoted as $\Inf_i(f)$ is given by the quantity $2\Prx_{\bx \sim \mcD_X, \bb \sim \mcD_X^{(i)}}[f(\bx)\ne f(\bx_{i = \bb})]$, where $\bx_{i = \bb}$  rerandomises the $i$-th bit of $\bx$ with a random sample from $\mcD_X^{(i)}$. 
\end{definition} 

\subsection{Impurity-based decision tree learning algorithms}
Essentially almost all decision tree learning algorithms used in practice, including the classic and popular ID3, C4.5, and CART, learn decision trees greedily in a top-down manner, using an {\sl impurity function} as a measure of progress. 

\begin{definition}[Impurity function]\label{def:impurity_function}
An impurity function $\G : [0,1] \to [0,1]$ is a concave function that is symmetric around $\frac1{2}$ and satisfies $\G(0) = \G(1) = 0$ and $\G(\frac1{2}) = 1$.
\end{definition}

\Cref{def:impurity_function} guarantees that $\G(p) \geq \min\{p, 1-p\}$ for all $p\in [0,1]$, allowing us to view $\G(\mu(\mcD_\ell))$ as an upper bound on $\eps(\mcD_\ell)$. If we analogously define $\G_\mcD(T) \coloneqq \E_{\bfell \sim \mcD}  [\G(\mu(\mcD_{\bfell}))]$, we get that $\G_\mcD(T)$ is an upper bound on $\error_\mcD[T]$. Common examples of impurity functions include $\G(p) = \textnormal{H}_2(p)$ (binary cross-entropy, used by ID3 and C4.5), $\G(p) = 4p(1-p)$ (Gini impurity function, or simply variance, used by CART), or $\G(p) = 2\sqrt{p(1-p)}$ (introduced and analyzed in~\cite{KM96}). For this paper, we will focus on $\G(p) = 4p(1-p)$ for simplicity, but our results hold generally for all impurity functions that have a second derivative bounded away from $0$ (see \Cref{remark:other-impurity} in appendix).

Impurity-based decision tree learning algorithms are parameterized by an impurity function $\G$ and a class $\mathcal{H}$ of allowable splitting functions. For a tree $T$, leaf $\ell$, and label function $h \in \mcH$, we denote $T_{\ell, h}$ to be the extension of $T$ that replaces the leaf $\ell$ with an internal node that splits on $h$.   At every step, the algorithm loops through all possible leaves and labelling functions $h \in \mcH$,\footnote{In the context of  practical decision tree algorithms, the class of splitting functions $\mcH$ is usually finite and small. Standard implementations of these popular algorithms, such as in scikit-learn, do a brute force search over hypotheses. } looking for a potential split that will result in the largest reduction $\G(T) - \G(T_{\ell, h})$ of the impurity function, and hence also (hopefully) the error of the new tree.

\begin{definition}[Purity gain]\label{def:purity_gain}
Let $h$ be a splitting function, $\ell$ be a leaf of $T$, and $\mcD_\ell$ be $\mathcal{D}$ conditioned on $\bx$ reaching $\ell$. Let $\ell_0$ be the leaf of $T_{\ell, h}$ corresponding to $h(\bx) = 0$, $\ell_1$ be the leaf corresponding to $h(\bx) = 1$, and $\mcD_{\ell_0}$ and  $\mcD_{\ell_1}$ be their respective conditional distributions. 

We define the local drop in $\G$ at this leaf, after splitting with $h$, as:
\begin{align*}
    \Delta_{\mcD_\ell}(h) &\coloneqq \G(\mu(\mcD_\ell)) -\Prx_{\mcD_\ell}[h(\bx) = 0] \cdot \G(\mu(\mcD_{\ell_0}))  \\
    & \quad  - \Prx_{\mcD_\ell}[h(\bx) = 1] \cdot \G(\mu(\mcD_{\ell_1})).
\end{align*} 
\end{definition}
\begin{figure}[ht]
\begin{tcolorbox}[colback = white,arc=1mm, boxrule=0.25mm]
\vspace{3pt}
$\BuildTD_{\mcH, \G, \mcD}(t)$:
\begin{itemize}[align=left]
	\item[\textbf{Given:}] Size bound $t$. 
	\item[\textbf{Output:}] Decision tree of size $t$ approximating $\mcD$, with internal nodes chosen from $\mcH$.  
\end{itemize}
\begin{enumerate}
	\item Initialize $T$ to be the empty tree.
	\item While $\mathrm{size}(T) < t$: 
	\begin{enumerate}
	    \item Let $(\ell^\star, h^\star)$ be set to $$\argmax_{(\ell, h) \in \mathrm{leaves}(T) \times \mcH} [w_\mcD(\ell) \cdot  \Delta_{\mcD_\ell}(h)]$$ breaking ties arbitrarily.
	    
	    \item Set $T \leftarrow T_{\ell^*, h^*}$.

	\end{enumerate}
	\item Label each leaf $\ell \in \mathrm{leaves}(T)$ with value 
	$$\Ind[\mu(\mcD_\ell) > \lfrac{1}{2}].$$
	\item Return $T$.
\end{enumerate}

\end{tcolorbox}
\caption{The decision tree boosting algorithm for hypothesis class $\mcH$ and impurity function $\mcG$ over distribution $\mcD$. For simplicity, we assume that $\Delta_{\mcD_\ell}(h)$, $w_\mcD(\ell)$, and $\mu(\mcD_\ell)$ can be computed exactly. In practice, these quantities can be replaced with empirical estimates from a random sample of $\mcD$ (see \Cref{apx:exact_vs_sample} for details). 
}
\label{fig:pseudocode}
\end{figure}


\subsection{Impurity-based decision tree algorithms as boosting algorithms}

\citet{KM96}  were the first to analyze impurity-based decision tree learning algorithms from the perspective of {\sl boosting}. Their simple but key insight was that the splitting functions at the internal nodes of the tree can be viewed as {\sl weak hypotheses}, and the decision tree construction as a process of creating a strong learner by combining these weak hypotheses. We recall that standard weak learning assumption from the literature on boosting:  


\begin{definition}[Standard distribution-independent weak learning assumption]
\label{def:standard-WLA} 
Let $f : \mathcal{X} \to \zo$ be a target function and $\mathcal{H}$ be a class of hypotheses from $\mathcal{X}$ to $\zo$. For $\gamma > 0$, we say that {\sl $\mathcal{H}$ satisfies the distribution-independent $\gamma$-weak learning assumption w.r.t.~$f$} if for all distributions $\mathcal{D}_X$ over $\mathcal{X}$, there exists $h\in\mathcal{H}$ such that $\Prx_{\bx\sim\mathcal{D}_X}[h(\bx)\ne f(\bx)]\le \frac1{2}-\gamma$.
\end{definition}

\Cref{def:standard-WLA} can be hard to satisfy because of the requirement that there exists a $\gamma$-advantage hypothesis {\sl for every distribution} $\mathcal{D}_X$ over $\mathcal{X}$.  Our analysis will only rely on a  milder {\sl distribution-specific} weak learning assumption that is significantly easier to satisfy.  Looking ahead, the mildness of our weak learning assumption will be crucial for our proof of~\Cref{thm:monotone-intro}. We first need the notion of an induced distribution:

\begin{definition}[Induced distributions]
Let $\mathcal{D}$ be a distribution over $\mathcal{X}\times \zo$ and $\mathcal{H}$ be a class of hypotheses from $\mathcal{X}$ to $\zo$.  We say that $\mathcal{D'}$ is a distribution {\sl induced by conditioning $\mathcal{D}$ on $\mathcal{H}$} if $\mathcal{D}'$ can be expressed as $\mathcal{D}$ conditioned on $\bx \sim \mathcal{D}_X$ satisfying $h_1(\bx)\wedge \cdots \wedge h_k(\bx)$ where $h_i \in \mathcal{H}$.  
\end{definition}

\begin{remark}
Note that all the conditional distributions, $\mcD_\ell$, at the leaves of a decision tree are induced distributions of $\mcD$ conditioned on $\mcH$. Moreover, if $\mathcal{D}$ is such that $\mathcal{D}_X$ is a product distribution and if $\mathcal{H}$ is the class of hypotheses that threshold on a single feature (i.e.~$h(x) = \Ind[x_i\ge \theta]$), then $\mathcal{D}'_X$ remains a product distribution for every distribution $\mathcal{D}'$ that is induced by conditioning $\mathcal{D}$ on $\mathcal{H}$.
\end{remark}

\begin{definition}[Our distribution-specific weak learning assumption]
\label{def:our-WLA} 
Let $\mcD$ be a distribution over $\mcX \times \{0, 1\}$ and $\mcH$ be a class of hypotheses from $\mcX$ to $\{0,1\}$. For $\gamma > 0$,  we say {\sl $\mcH$ satisfies the $\gamma$-weak learning assumption w.r.t $\mcD$} if, for any distribution $\mathcal{D}'$ that is induced by conditioning $\mathcal{D}$ on $\mathcal{H}$, there exists an $h \in \mcH$ that satisfies: 
\begin{equation}\label{eq:weak_learning_req}
    \left|\Cov_{\mcD'}[h(\bx), \by]\right| \geq \gamma \Var_{\mcD'}[\by].
\end{equation}
We call such an $h$ a $\gamma$-advantage hypothesis with respect to $\mathcal{D'}$. 
\end{definition}

\begin{remark}
Our weak learning assumption in  \Cref{def:our-WLA} is a weaker assumption than the standard \Cref{def:standard-WLA}. This is because \Cref{eq:weak_learning_req} is equivalent to having $\Pr_{(\bx,\by) \sim \mcD'_{\text{bal}}} [h(\bx) \neq \by] \leq 1/2 - \gamma$, where $\mcD'_\text{bal}$ is the {\em{balanced}} version of $\mcD'$, so that the points in $\mcX$ are reweighted to make the probability of a positive or negative label equally likely. Since \Cref{def:standard-WLA} has to hold for \emph{all} distributions over $\mcX$, it also has to hold in particular for the  marginal distribution of $\mcD'_\text{bal}$ over $\mcX$. Therefore, \Cref{def:standard-WLA} implies \Cref{def:our-WLA}.
\end{remark}




\subsection{Learning with adversarial noise}
\label{prelims:noise} 

 Adversarial noise can take on many strengths and forms, depending on both what kind of corruptions are allowed and when these corruptions can be made. We focus on the strongest model of noise called \emph{nasty noise} \cite{BEK02}. In this setting, we wish to learn a binary classifier on a distribution $\mcD$ over $\mcX \times \{0, 1\}$.  However, instead of receiving a set of samples $S \sim \mcD^n$, an adversary is allowed to replace any $\eta$-fraction of points in $S$ with arbitrary points to get a corrupted sample $\wh{S}$. The algorithm then receives the corrupted sample $\wh{S}$.

This noise model captures many other weaker forms of noise. For example, if the adversary can only change the {\sl labels} of the $\eta$ corrupted fraction, we recover agnostic noise. Similarly, if the adversary has to obliviously commit to a corruption strategy before seeing the sample $S$, this is equivalent to choosing a distribution $\wh{\mcD}$ that is $\eta$-close to $\mcD$ in Total Variation (TV) distance, and drawing $\wh{S}$ from that. This is often called the distributional shift setting. 



\section{Proof of~\Cref{thm:main}: Impurity-based decision tree algorithms are noise-tolerant boosting algorithms}

We can now state the formal version of \Cref{thm:main}, which states that impurity-based decision tree learning algorithms are boosting algorithms that are resilient to nasty noise. 

We draw on a recent result by \citet{BLMT21}, which shows that learning in the presence of $\eta$-nasty-noise corruption is equivalent to learning in the presence of $\eta$ distribution shift with respect to Total Variation distance ($\dtv$), as long as the learner only interacts with its samples by computing expectations.
For details on the formal relationship between \Cref{thm:main-intro,thm:main_formal}, as well as the runtime analysis of $\BuildTD$ see \Cref{apx:exact_vs_sample}.

\begin{theorem}[Formal version of~\Cref{thm:main-intro}]\label{thm:main_formal}
Let $\mcD$ be a distribution over $\mcX \times \{0, 1\}$ and $\mcH$ be a class of splitting functions from $\mathcal{X}$ to $\zo$  that satisfies the $\gamma$-weak learning assumption w.r.t.~$\mcD$. For any noise rate $\eta \le O(\eps\gamma)$ and any distribution $\wh{\mcD}$ satisfying $\dtv(\mcD, \wh{\mcD}) \leq \eta$, the decision tree hypothesis $T$ constructed by $\BuildTD_{\mcH, \G, \wh{\mcD}}(t)$ achieves $\error_{\mcD}[T] \leq \eps$ after $t \geq  \exp\paren{O(1/\gamma^2 \eps^2)}$. 
\end{theorem}

We now prove a few lemmas that will be useful for our proof of~\Cref{thm:main_formal}. We begin by quantifying how much the impurity function decreases when we split at this particular leaf with a splitting function $h : \mathcal{X} \to \zo$.  The following lemma lower bounds this decrease, $\Delta_{\wh{\mcD}_\ell}(h)$,  in terms of the covariance between $h(\bx)$ and $\by$.  We state the lemma generally for an arbitrary distribution $\mcE$ since we will be applying it to the distributions, ${\wh{\mathcal{D}}_\ell}$, at each leaf.


\begin{lemma}[Local drop in $\G$ in terms of covariance]\label{lem:delta_impurity}
Let $\mcE$ be any distribution over $\mathcal{X}\times \zo$ and  $h : \mcX \to \{0,1\}$ be a splitting function. Then:
\begin{equation}
    \label{eq:delta-drop-16}
    \Delta_{\mcE}(h) \geq 16 \cdot \Cov_{\mcE}[h(\bx), \by]^2.
\end{equation}
\end{lemma}

\begin{proof}
This result follows almost directly from \citet{KM96}. See \Cref{apx:delta_impurity} for details.
\end{proof}

Our weak learning assumption (\Cref{def:our-WLA}) provides lower bounds on the covariance between $h(\bx)$ and $\by$ on the {\sl uncorrupted} distribution $\mcD_{\ell}$. 
We want to relate this  this to the covariance on an adversarially corrupted distribution $\wh{\mcD}_{\ell}$ that is $\eta_{\ell}$-close to $\mcD_{\ell}$. 


We will need the following useful facts relating the variance and covariance of bounded functions on two distributions that are close in TV distance. We consider an arbitrary domain $\mcV$ and, without loss of generality, functions from $\mcV$ to $[0,1]$. We defer the proofs to \Cref{apx:moments_tv_dist}.

\begin{lemma}[Moments and TV-distance]\label{lem:moments_tv_dist}
Let $\mcE, \wh{\mcE}$ be two distributions over a domain  $\mcV$ with $\dtv(\mcE, \wh{\mcE}) \leq \eta$ and let $f,g : \mcV \to [0,1]$ be functions. Then 
\begin{align}
     | \Var_{\mcE}[f(\bx)] &- \Var_{\wh{\mcE}}[f(\bx)]  | \leq \eta \label{eq:var_tv} \\ 
     | \Cov_{\mcE}[f(\bx), g(\bx)] &- \Cov_{\wh{\mcE}}[f(\bx), g(\bx)]  | \leq  2\eta \label{eq:cov_tv}.
\end{align}
\end{lemma}


We now use \Cref{lem:moments_tv_dist} with our weak learning assumption to prove the existence of an $h$ at that has high covariance on the adversarially corrupted distribution at a given leaf:

\begin{lemma}[Covariance on the corrupted distribution] \label{lem:delta_leaf}
Let $\mcD$ be a distribution over $\mcX \times \{0, 1\}$ and $\mcH$ be a class of splitting functions from $\mathcal{X}$ to $\zo$  that satisfies the $\gamma$-weak learning assumption w.r.t $\mcD$.  For a decision tree $T$ and leaf $\ell$ of $T$, let $\wh{\mcD}_\ell$ be any distribution with $\dtv(\mcD_\ell, \wh{\mcD}_\ell) \leq \eta_\ell$. Then there is an $h_\ell \in \mcH$ s.t.
\begin{equation*}
    | \Cov_{\wh{\mcD}_\ell}[h_\ell(\bx), \by] | \geq
    \gamma \Var_{\wh{\mcD}_\ell}[\by] - 3  \eta_\ell.
\end{equation*}
\end{lemma}
\begin{proof}
Since $\mathcal{H}$ satisfies the weak learning assumption w.r.t.~$\mathcal{D}$, and since $\mcD_\ell$ is an induced distribution of $\mcD$, there exists an $h_\ell \in \mcH$ s.t. $|\Cov_{\mcD_\ell}[h_\ell(\bx), \by]| \geq \gamma \Var_{\mcD_\ell}[\by]$.

By \Cref{lem:moments_tv_dist}, since  $\dtv(\mcD_\ell, \wh{\mcD}_\ell) \leq \eta_\ell$, we have:
\begin{align*} 
   |\Cov_{\wh{\mcD}_\ell}[h_\ell(\bx), \by]|
        &\geq |\Cov_{\mcD_\ell}[h_\ell(\bx), \by]| - 2 \eta_\ell \tag{\Cref{lem:moments_tv_dist}} \\ 
                    &\geq  \gamma \Var_{\mcD_\ell}[\by] - 2   \eta_\ell \tag{Weak learning assumption} \\
                    &\geq \gamma \Var_{\wh{\mcD}_\ell}[\by] - 3   \eta_\ell. \tag{\Cref{lem:moments_tv_dist}}
\end{align*}
\end{proof}
We can now prove the theorem:
\begin{proof}[Proof of \Cref{thm:main_formal}]
If $\error_{\wh{\mcD}}[T] < \frac{12\eta}{\gamma} + \eps$,  we are done since it follows that $\error_{\mcD}[T] < \eta + \frac{12\eta}{\gamma} + \eps \le O(\eps)$, by our assumption that $\eta \le O(\eps\gamma)$.

Otherwise, if $\error_{\wh{\mcD}}[T] \ge \frac{12\eta}{\gamma} + \eps$, we  prove the existence of a leaf $\ell^* \in T$ and splitting function $h_{\ell^*} \in \mcH$ such that splitting $\ell^*$ according to  $h_{\ell^*}$ results in a substantial reduction in $\G_{\wh{\mathcal{D}}}(T)$.  In more detail, we consider the expected reduction in $\G_{\wh{\mathcal{D}}}(T)$ if a {\sl random} leaf $\bell\sim (T,\wh{\mathcal{D}})$ is split with the respective $h_{\bell}$ from \Cref{lem:delta_leaf}:

\begin{align*}
    \Ex_{\bell \sim (T,\wh{\mcD})}[\Delta_{\wh{\mcD}_{\bell}}(h_{\bell})] 
        &\geq 16 \cdot 
            \Ex_{\bell \sim (T,\wh{\mcD})}
            \left[
                \Cov_{\wh{\mcD}_\bell}[h_\bell(\bx), \by]^2
            \right] 
            \tag{\Cref{lem:delta_impurity}} \\
        &\geq 16 \cdot 
            \E_{\bell \sim (T,\wh{\mcD})}
            \left[
                |\Cov_{\wh{\mcD}_\bell}[h_\bell(\bx), \by]|
            \right]^2 
            \tag{Jensen's inequality} \\
        &\geq 16 \cdot 
            \relu{\E_{\bell \sim (T,\wh{\mcD})}
            \left[
                |\Cov_{\wh{\mcD}_\bell}[h_\bell(\bx), \by]|
            \right] }^2 
            \tag{Since $x^2 \geq \relu{x}^2$} \\
        &\geq 16 \cdot 
            \relu{\E_{\bell \sim (T,\wh{\mcD})}
            \left[ 
                \gamma \Var_{\wh{\mcD}_{\bell}}[\by] - 3   \eta_{\bell}
            \right]}^2 
            \tag{\Cref{lem:delta_leaf}} \\
        &= 16 \cdot 
            \relu{
                \gamma \E_{\wh{\mcD}}[\Var_{\wh{\mcD}_{\bell}}[\by]] - 3   \E_{\wh{\mcD}}[\eta_{\bell}]
            }^2,
            \tag{Linearity of expectation} 
\end{align*}

where we use the notation $\relu{x} \coloneqq \max\{0, x\}$. 

Since $\Var_{\wh{\mcD}_{\bell}}[\by] = \mu(\wh{\mcD}_{\bell})(1 - \mu(\wh{\mcD}_{\bell}))$, it is clear that $\Var_{\wh{\mcD}_{\bell}}[\by] \geq 1/2 \cdot  \eps(\wh{\mcD}_{\bell})$. Therefore 
\begin{equation}\label{eq:var_geq_err}
    \Ex_{\bell \sim (T, \wh{\mcD})} [\Var_{\wh{\mcD}_{\bell}}[\by]] \geq \error_{\wh{\mcD}}[T]/2.
\end{equation}

All together, we have
\begin{align*}
    \Ex_{\bell \sim (T,\wh{\mcD})}[\Delta_{\wh{\mcD}_{\bell}}(h_{\bell})] 
        &\geq 16 \cdot 
            \relu{
                \gamma \E_{\wh{\mcD}}[\Var_{\wh{\mcD}_{\bell}}[\by]] - 3   \E_{\wh{\mcD}}[\eta_{\bell}]
            }^2 \\
        &\geq 16 \cdot 
            \relu{
                \frac{\gamma}{2} \cdot \error_{\wh{\mcD}}[T] - 
                3\E_{\wh{\mcD}}[\eta_{\bell}]
            }^2  
            \tag{\Cref{eq:var_geq_err}}  \\
        &\geq 16 \cdot 
            \relu{
                \frac{\gamma}{2} \cdot \error_{\wh{\mcD}}[T] - 
                6\eta
            }^2  
            \tag{$\E_{\wh{\mcD}}[\eta_{\bell}] \leq 2\eta$ by \Cref{lem:TV-leaves}}\\
        &\geq 16 \cdot \relu{\frac{\gamma \eps}{2} + 6 \eta -  6\eta}^2 \tag{by assumption} \\
        &\geq 4 \gamma^2 \eps^2.
\end{align*}

Rewriting the expectation, we get
\begin{equation*}
    \sum_{\ell \in T} w_{\wh{\mcD}}(\ell^*) \Delta_{\wh{\mcD}_{\ell^*}} (h_{\ell^*}) \geq 4 \gamma^2 \eps^2.
\end{equation*}
If there are currently $s$ leaves in $T$, there must exist a leaf $\ell^\star$ such that
\begin{equation*}
    w_{\wh{\mcD}}(\ell^*) \Delta_{\wh{\mcD}_{\ell^*}} (h_{\ell^*}) \geq \frac{4 \gamma^2 \eps^2}{s}.
\end{equation*}
Since $\textsc{TopDownDT}$ greedily splits the leaf that results in the largest drop in $\G_{\wh{\mathcal{D}}}(T)$, it follows that the drop in $\G$ at timestep $s$ is at least least $4 \gamma^2 \eps^2 / s$. Therefore,  after $t$ steps, the total drop is at least:
\begin{align*}
    4 \gamma^2 \eps^2 \paren{1 + \frac{1}{2} + \frac{1}{3} + \ldots + \frac{1}{t}} \geq  4 \gamma^2\eps^2  \log t.
\end{align*}

Since the range of $\G_{\wh{\mathcal{D}}}$ is $[0,1]$, we conclude that after $t = \exp(O(1/\gamma^2\eps^2))$ steps, we must be done i.e. $\error_{\wh{\mcD}}[T] < \lfrac{12\eta}{\gamma} + \eps$, and hence $\error_{\mathcal{D}}[T]\le O(\eps)$, since $\eta \le O(\eps \gamma )$.
\end{proof}

\section{Proof of \Cref{thm:lb-intro}: Optimality of our parameters} 

\Cref{thm:main_formal} says that $\BuildTD$ can grow a tree with error $\leq \eps$ only when $\eta \leq O(\eps \gamma)$, where $\eta$ is the amount of corruption and $\gamma$ is the weak learning advantage. Here, we show that this bound is tight: If we allow $\eta = \tilde{O}(\eps \gamma)$, then we can design distributions $\wh{\mathcal{D}}$ on which $\BuildTD$ fails to achieve error $\leq \eps$.

We remark that for technical convenience, in both this section and \Cref{sec:monotone}, we switch to functions outputting $\bits$ rather than $\zo$. The two formulations are equivalent.


\begin{theorem}[Formal version of \Cref{thm:lb-intro}]
    \label{thm:lb}
    For any $\eps, \gamma >0$ where $\gamma^{1/\gamma} \leq \eps$, $d \in \N$, and $\eta \geq \Omega(\gamma \eps \log(1/\eps))$. There is a distribution $\mcD$ whose marginal over $\mcX \coloneqq \bits^d$ is uniform, an $\eta$-nasty noise corruption $\wh{\mcD}$ of $\mcD$, and a hypothesis class, $\mcH$, satisfying the $\gamma$-weak learning assumption w.r.t $\mcD$, such that for all impurity function $\mcG$, $\BuildTD_{\mcH, \mcG, \wh{\mcD}}(t)$ fails to build an $\eps$-error tree for $\mcD$ unless $t \geq 2^{d - O(\log(1/\gamma)/\gamma)}$.
\end{theorem}

\paragraph{Proof sketch.} We will design a function $f:\bits^d \to \bits$ that only depends on its first $k$ features (meaning $f(x) = g(x_{[1:k]})$ for some function $g: \bits^k \to \bits$) for $k \ll d$ and set $\mcD$ to the distribution of $(\bx, f(\bx))$ where $\bx$ is uniform from $\bits^d$. This function will be carefully designed so that there is an $\eta$-corruption $\wh{\mcD}$ of $\mcD$ in which \emph{every} hypothesis $h \in \mcH$ has a local drop in $\mcG$ of $0$. As a result, $\BuildTD$ cannot identify the ``important" $k$ hypotheses and is likely to pick one of the $d - k$ useless (because they are independent of the label) hypothesis. That continues until all $d-k$ useless hypotheses have been used, which requires the tree to have depth $d-k$ corresponding to a size of $2^{d-k}$.

To formalize the above proof sketch, we will need to prove that the $\mcD$ we design satisfies the weak-learning hypothesis. To do so, we use the celebrated Kahn–Kalai–Linial inequality from the analysis of Boolean functions.
\begin{fact}[KKL inequality, \cite{KKL88}]
    \label{fact:kkl}
    For any function $f:\bits^k \to \bits$ and $\mathcal{D}_X$ the uniform distribution over $\bits^k$, there is a coordinate $i \in [k]$ for which
    \begin{equation*}
        \Inf_i(f) \geq \Omega\left(\frac{\log k}{k} \cdot \Varx_{\bx \sim \mcD_X}[f(\bx)]\right).
    \end{equation*}
\end{fact}
The KKL inequality will allow us to prove that a broad class of distributions satisfy the weak-learning assumption.
\begin{definition}[Monotone functions]
    \label{def:monotone}
    We say that a Boolean function $f:\bits^n \to \bits$ is \emph{monotone} if for any $x,y \in \bits^n$ where $x_i \leq y_i$ for all $i \in [n]$, $f(x) \leq f(y)$.
\end{definition}

Combining \Cref{fact:kkl} with \Cref{lem:inf-corr} immediately gives the following.
\begin{corollary} 
    \label{cor:monotone-cov}
    For any monotone function $g: \bits^k \to \bits$, there is a coordinate $i \in [k]$ satisfying, for $\bx$ uniform from $\bits^k$
    \begin{equation*}
        \Cov\left[\bx_i, g(\bx) \right] \geq \Omega\left(\frac{\log k}{k} \cdot \Var[g(\bx)]\right).
    \end{equation*}
\end{corollary}

We apply the above corollary to prove a class of distributions satisfying the weak-learning assumption. Before doing so, we'll need the notion of a \emph{restriction}: 

\paragraph{Restrictions.}
Given some domain $\mcX \coloneqq \bits^d$, a \emph{restriction} of that domain is a defined by value for each coordinate, $\rho \in \{-1, +1, \star\}^d$. An input $x \in \mcX$ is said to be \emph{consistent} with a restriction $\rho$, if $x_i = \rho_i$ for all $i \in [d]$ where we define $\star$ to be equal to both $+1$ and $-1$. The coordinates $i$ where $\rho_i \in \bits$ are said to be \emph{specified}, and we define $|\rho|$ to be the number of coordinates specified. The number of inputs consistent with a restriction $\rho$ is $2^{d - |\rho|}$, and there is a natural projection $\proj_{\rho}$ from $\bits^{d - |\rho|}$ to the subset of $\mcX$ consistent with $\rho$ that uses the input to $\proj_{\rho}$ for the unspecified coordinates and fills in the specified coordinates according to $\rho$.

\begin{proposition}
    \label{prop:junta-wl}
    For any $k \leq d$ and monotone function $g: \bits^k \to \bits$, let $f:\bits^d \to \bits$ be the function that computes $f(x) = g(x_{[1:k]})$ where $x_{[1:k]}$ is the first $k$-bits of $x$ and $\mathcal{D}$ be the distribution over $(\bx, f(\bx))$ where $\bx$ is uniform in $\mcX \coloneqq \bits^d$. Then, the set of coordinate projections, $\mathcal{H} \coloneqq \{\mathrm{proj}_i: i \in [d]\}$, satisfies the $(\gamma = O((\log k)/k))$-weak learning assumption (\Cref{def:our-WLA}) w.r.t. $\mcD$.
\end{proposition}
\begin{proof}
    Let $\mcD'$ be any induced distribution of $\mcD$ by $\mcH$. Our goal is to show that \Cref{eq:weak_learning_req} is satisfied, or equivalently, that there is some $i \in [d]$ for which
    \begin{equation*}
        \Cov_{\bx \sim \mcD_X'}[\bx_i, f(\bx)] \geq \gamma \Var_{\bx \sim \mcD_X'}[f(\bx)].
    \end{equation*}
    As $\mcH$ is the set of coordinate projections and $\mcD_X$ is uniform over $\bits^d$, every induced distribution $\mcD'_X$ corresponds to the uniform distribution over all elements of $\mcX$ consistent with some restriction $\rho$. Given that restriction $\rho$, we can define the function $f_{\rho}: \bits^{d - |\rho|} \to \bits$ defined $f_{\rho}(x) = f(\proj_{\rho}(x))$.
    
    As $f$ is a monotone function of the first $k$ coordinates of its input, $f_{\rho}$ is a monotone function of the first (up to $k$) coordinates of its input. Hence, there is some $k' \leq k$ and monotone $g_{\rho}:\bits^{k'} \to \bits$ for which $f_\rho(x) = g_\rho(x_{[1:k']})$. Then,
    \begin{align*}
        \max&_{i \in [d]} \left(\underset{\bx \sim \mcD_X'}{\Cov}[\bx_i, f(\bx)] \right)  \\
        &= \max_{i \in [d - |\rho|]} \left(\underset{\bx \sim  \bits^{d - |\rho|}}{\Cov}[\bx_i, f_\rho(\bx)] \right) \\
        &= \max_{i \in [k']} \left(\underset{\bx \sim  \bits^{k'}}{\Cov}[\bx_i, g_\rho(\bx)] \right) \\
        &\geq \Omega\left(\frac{\log k'}{k'} \cdot \Varx_{\bx \sim \bits^{k'}}[g_{\rho}(\bx)]\right) \tag{\Cref{cor:monotone-cov}}\\
         &\geq \gamma \Var_{\bx \sim \mcD_X'}[f(\bx)] \tag{since $k' \leq k$, $\gamma = O((\log k)/k)$}.
    \end{align*}
    This means that $\mcH$ satisfies our weak learning assumption w.r.t. $\mcD$.
\end{proof}

 To design the distribution $\mcD$ of \Cref{thm:lb}, we'll use the following proposition.
  \begin{proposition}
    \label{prop:biased-f}
    For any $\eps \in (0, 1/3]$ and $d \geq \log_2(1/\eps)$, for some integer $k \leq d$, there is a monotone function $f: \bits^k \to \bits$ where, for $\bx \sim \bits^k$ chosen uniformly,
    \begin{equation*}
        \min_{b \in \bits}\Pr[f(\bx) = b] \geq \eps
    \end{equation*}
    and,
    \begin{equation}
        \Ex[\bx_1 f(\bx)] = \cdots = \Ex[\bx_k f(\bx)] = O\left(\eps \log\paren{\frac{1}{\eps}} \cdot \frac{\log d}{d}\right).
    \end{equation}
\end{proposition}
The proof of \Cref{prop:biased-f} is given in \Cref{apx:lb}.

Finally, we prove the main result of this section.
\begin{proof}[Proof of \Cref{thm:lb}]
    Let
    \begin{equation*}
        \ell \coloneqq \ceil*{O\left(\frac{\log(1/\gamma)}{\gamma}\right)}.
    \end{equation*}
    Note as we assume that $\gamma^{1/\gamma} \leq \eps$, we have that $\ell \geq \log_2(1/\eps)$. Therefore, by \Cref{prop:biased-f}, we know for some $k \leq \ell$, there exists a monotone $g: \bits^k \to \bits$ where for $\bx \sim \bits^k$ chosen uniformly,
    \begin{equation}
    \label{eq:high-bias}
        \min_{b \in \bits}\Pr[g(\bx) = b] \geq 2\eps
    \end{equation}
    and,
    \begin{equation}\label{eq:correlations}
        \begin{aligned}
        v &\coloneqq \Ex[\bx_1 g(\bx)] = \cdots = \Ex[\bx_k g(\bx)] \\
        &\quad = O\left(\eps \log(1/\eps) \cdot \frac{\log \ell}{\ell}\right) \\
        &\quad = O\left(\eps \log(1/\eps) \cdot \gamma\right).
        \end{aligned}
    \end{equation}
    
    Let $f:\bits^d \to \bits$ be the function that computes $f(x) = g(x_{[1:k]})$, and let $\mcD$ be the distribution over $(\bx, f(\bx))$ where $\bx$ is uniform in $\mcX \coloneqq \bits^d$. Then, by \Cref{prop:junta-wl}, the class of coordinate projection, $\mathcal{H} \coloneqq \{\mathrm{proj}_i: i \in [d]\}$, satisfies the $\gamma$-weak learning assumption (\Cref{def:our-WLA}) w.r.t. $\mcD$.
    
    Next, we define the corrupted distribution $\wh{\mcD}$. Let $\mcE$ be the distribution where, to sample $(\bx, \by) \sim \mcE$, we first draw $\by$ uniformly in $\bits$ and $\bz$ uniformly in $\bits^{d-k}$. Then, we set $\bx$ to be
    \begin{equation*}
        \bx = (\underbrace{-\by, \ldots, -\by}_{\text{$k$ copies}}) \circ \bz
    \end{equation*}
    where $\circ$ represents concatenation. Then, we set the corrupted distribution to be the mixture $\wh{\mcD} \coloneqq (1 - \eta) \mcD + \eta \mcE$
    where $\eta$ is chosen as the unique solution of $(1-\eta)v - \eta = 0$ where $v$ is as defined in \Cref{eq:correlations}. Note that this solution satisfies $\eta \leq v = O(\eps \log(1/\eps)\gamma)$. As $\wh{\mcD}$ is a mixture with $(1-\eta)$ fraction coming from $\mcD$, it is an $\eta$-nasty noise corruption of $\mcD$. Furthermore, as the contribution of $\mcE$ is chosen to exactly cancel out the correlations of $\mcD$, for all $i \in [d]$,
    \begin{equation*}
        \E_{\wh{\mcD}}[\by] = \E_{\wh{\mcD}}[\by  \mid \bx_i = -1] = \E_{\wh{\mcD}}[\by  \mid \bx_i = +1] = 0.
    \end{equation*}
    
    This means that for \emph{any} impurity function $\mcG$, all hypotheses have a local drop in $\mcG$ of $0$. Furthermore, all projections except for the first $k$ are fully independent of one another and of the label. Therefore, if all internal nodes in the tree consists of projections for the last $d-k$ coordinates, then all hypotheses at every leaf will still have impurity gain.
    
    As a result, $\BuildTD$ will choose an arbitrary $(\ell^\star, h^\star)$ at each iteration. Unless $t \geq 2^{d - k}$, these arbitrary decisions can lead to the complete tree of depth $\log(t)$ being built, where all internal nodes have a hypothesis for one of the $d-k$ projection functions that are independent of $\by$.
    
    In that case, by \Cref{eq:high-bias} for every leaf $\ell$, $ \min_{b \in \bits}\Pr_{\mcD_\ell}[g(\bx) = b] \geq 2\eps$. Therefore, regardless of how $\BuildTD$ labels the leaves, the resulting tree will have error $\geq 2\eps$.
\end{proof}

\section{Proof of \Cref{thm:monotone-intro}: Learning monotone decision trees in the presence of nasty noise}
\label{sec:monotone}

In this section we consider distributions $\mathcal{D}$ for which the marginal $\mathcal{D}_X$ is an arbitrary product distribution $\mathcal{D}_X = \mcD_X^{(1)} \times \ldots \times \mcD_X^{(d)}$ over $\bits^d$ (i.e. each bit is independent) and the deterministic target function $\mathcal{D}_{Y|X}$ is monotone and representable by a size-$s$ decision tree.  



The following is a useful fact about the influence of features on {\sl monotone} functions. See \Cref{apx:inf-corr} for a proof.



\begin{lemma}[Influence = covariance for monotone functions]
\label{lem:inf-corr}
Let $\mcD_X = \mcD_X^{(1)} \times \ldots \times \mcD_X^{(d)}$ be an arbitrary product distribution over $\bits^d$. For a monotone function $f : \{\pm 1\}^d \to \bits$ and a feature $i\in [d]$, we have the identity $\Inf_i(f) =  \Cov_{\mcD_X}[f(\bx), \bx_i]$.
\end{lemma}

The key technical ingredient in our proof of~\Cref{thm:monotone-intro} is a theorem of O'Donnell, Saks, Schramm, and Servedio from discrete Fourier analysis~\cite{OSSS05}: 


\begin{theorem}[OSSS inequality]
\label{thm:OSSS}
Let $f : \bits^d \to \bits$ be  function that is representable by a size-$s$ decision tree and $\mcD_X$ be a product distribution over $\bits^d$. Then 
\begin{equation}
    \max_{i \in [d]} \{\Inf_i(f) \} \geq \frac{\Var[f]}{\log s},
\end{equation}
where $\Inf_i(f)$ and $\Var[f]$ are with respect to $\mcD_X$.
\end{theorem}

For the class of distributions described at the beginning of this section,~\Cref{lem:inf-corr} and~\Cref{thm:OSSS} together imply that the weak learning assumption of~\Cref{thm:main_formal} can be satisfied by the set $\mathcal{H} = \{\mathrm{proj}_i : i \in [d]\}$ of projection functions:  

\begin{lemma}[Projection functions satisfy weak learning assumption]
\label{lem:dt_wla} 
Let $\mathcal{D}$ be a distribution for which the marginal $\mathcal{D}_X$ is a product distribution over $\mathcal{X} = \bits^d$ and the target function $f \coloneqq \mathcal{D}_{Y|X}$ is monotone and can represented as a size-$s$ decision tree.  The set $\mathcal{H} = \{\mathrm{proj}_i : i \in [d]\}$ of projection functions satisfies the $\gamma$-weak learning assumption w.r.t.~$\mathcal{D}$ with $\gamma = 1/\log s$.
\end{lemma} 

\begin{proof} 
We first note that there is an $h\in \mathcal{H}$ that is a $\gamma$-advantage hypothesis with respect to $\mathcal{D}$: 
\[ \Cov[f(\bx),\bx_i] = \Inf_i(f) \ge \frac{\Var[f]}{\log s},\]
where we have used~\Cref{lem:inf-corr} for the equality and~\Cref{thm:OSSS} for the inequality.  Since $\mathcal{H}$ is the set of projection functions, every distribution $\mathcal{D}'$ that is induced by conditioning $\mathcal{D}$ on $\mathcal{H}$ is such that $\mathcal{D}'_X$ is a product distribution over $\bits^S$ for some $S \sse [d]$. Similarly, since $f$ is monotone and representable by a size-$s$ decision tree, the same remains true for any restriction of $f$ by the projection functions in $\mathcal{H}$.  Therefore, we can again apply~\Cref{lem:inf-corr} and~\Cref{thm:OSSS} to infer the existence of a $\gamma$-advantage hypothesis $h\in \mathcal{H}$ with respect to $\mathcal{D}'$.  
\end{proof} 



\Cref{thm:monotone-intro} is now an immediate consequence of~\Cref{thm:main_formal} and~\Cref{lem:dt_wla}: 

\begin{theorem}[Formal version of~\Cref{thm:monotone-intro}]\label{thm:monotone_dt}
Let $\mathcal{D}$ be a distribution for which the marginal $\mathcal{D}_X$ is a product distribution over $\mathcal{X} = \bits^d$ and the target function $\mathcal{D}_{Y|X}$ is monotone and can be represented as a size-$s$ decision tree.  For any impurity function $\G$, noise rate $\eta \le O(\eps/\log s)$, and distribution $\wh{\mathcal{D}}$ such that $\dtv(\wh{\mathcal{D}},\mathcal{D}) \le \eta$, the algorithm $\textsc{TopDownDT}_{\mathcal{H,\G,\wh{S}}}(t)$ where $t \coloneqq s^{O((\log s)/\eps^2)}$ runs in $\poly(d)\cdot s^{O((\log s)/\eps^2)}$ time and constructs a size-$t$ decision tree hypothesis $T$ satisfying $\error_\mathcal{D}[T]\le \eps.$
\end{theorem}

\section{Conclusion} 

We have given the first noise tolerance guarantees for the class of impurity-based decision tree learning algorithms that hold in a fully general setting.  \Cref{thm:main_formal} shows that they are noise-tolerant boosting algorithms that combine $\gamma$-advantage weak hypotheses into a strong hypothesis with error $\le \eps$, even in the presence nasty noise of rate as high as $\eta \le O(\eps\gamma)$. \Cref{thm:lb} provides a near-matching lower bound ruling out, in a strong sense, any such guarantee for noise rates $\eta \ge \tilde{\Omega}(\eps\gamma)$.  Finally, instantiating~\Cref{thm:main_formal} in the setting of product distributions over binary features---a setting that is particularly well studied in the theoretical literature---we show that these classic and widely-used algorithms achieve guarantees that are
better than those known for any existing theoretical algorithms.  Taken as a whole, our work helps place the popularity and empirical success of impurity-based decision tree learning algorithms on firm theoretical footing. 

There are several immediate avenues for future research. First, a natural next step is to establish similar formal noise tolerance guarantees for tree-based ensemble methods such as random forests and XGBoost.  Second, the focus of our work has been on understanding properties of impurity-based decision tree learning algorithms exactly as they are, to provide theoretical justification for their practical effectiveness.  It would nevertheless be interesting to consider possible modifications of these algorithms that are even more resilient to adversarial noise---for example, are there such modifications that evade our lower bounds?    


\section{Acknowledgments}
Guy and Li-Yang are supported by NSF CAREER Award 1942123. Jane is
supported by NSF Award CCF-2006664. Ali is supported by a graduate fellowship award from Knight-Hennessy Scholars at Stanford University.

\bibliography{ref}

\begin{thebibliography}{44}
\providecommand{\natexlab}[1]{#1}
\providecommand{\url}[1]{\texttt{#1}}
\expandafter\ifx\csname urlstyle\endcsname\relax
  \providecommand{\doi}[1]{doi: #1}\else
  \providecommand{\doi}{doi: \begingroup \urlstyle{rm}\Url}\fi

\bibitem[Blanc et~al.(2020{\natexlab{a}})Blanc, Lange, and Tan]{BLT-ICML}
Blanc, G., Lange, J., and Tan, L.-Y.
\newblock Provable guarantees for decision tree induction: the agnostic
  setting.
\newblock In \emph{Proceedings of the 37th International Conference on Machine
  Learning (ICML)}, 2020{\natexlab{a}}.

\bibitem[Blanc et~al.(2020{\natexlab{b}})Blanc, Lange, and Tan]{BLT-ITCS}
Blanc, G., Lange, J., and Tan, L.-Y.
\newblock Top-down induction of decision trees: rigorous guarantees and
  inherent limitations.
\newblock In \emph{Proceedings of the 11th Innovations in Theoretical Computer
  Science Conference (ITCS)}, volume 151, pp.\  1--44, 2020{\natexlab{b}}.

\bibitem[Blanc et~al.(2021{\natexlab{a}})Blanc, Lange, Qiao, and Tan]{BLQT21}
Blanc, G., Lange, J., Qiao, M., and Tan, L.
\newblock Decision tree heuristics can fail, even in the smoothed setting.
\newblock In Wootters, M. and Sanit{\`{a}}, L. (eds.), \emph{Proceedings of the
  25th International Conference on Randomization and Computation (RANDOM)},
  volume 207, pp.\  45:1--45:16, 2021{\natexlab{a}}.

\bibitem[Blanc et~al.(2021{\natexlab{b}})Blanc, Lange, Qiao, and
  Tan]{BLQT21focs}
Blanc, G., Lange, J., Qiao, M., and Tan, L.
\newblock Properly learning decision trees in almost polynomial time.
\newblock In \emph{Proceedings of the 62nd {IEEE} Annual Symposium on
  Foundations of Computer Science (FOCS)}, 2021{\natexlab{b}}.

\bibitem[Blanc et~al.(2021{\natexlab{c}})Blanc, Lange, and Tan]{BLT-ICALP}
Blanc, G., Lange, J., and Tan, L.-Y.
\newblock {Learning Stochastic Decision Trees}.
\newblock In \emph{Proceedings of the 48th International Colloquium on
  Automata, Languages, and Programming (ICALP)}, volume 198 of \emph{Leibniz
  International Proceedings in Informatics (LIPIcs)}, pp.\  30:1--30:16,
  2021{\natexlab{c}}.
\newblock ISBN 978-3-95977-195-5.

\bibitem[Blanc et~al.(2022)Blanc, Lange, Malik, and Tan]{BLMT21}
Blanc, G., Lange, J., Malik, A., and Tan, L.-Y.
\newblock On the power of adaptivity in statistical adversaries.
\newblock In \emph{Proceedings of the 35th Annual Conference on Computational
  Learning Theory (COLT)}, 2022.

\bibitem[Blum et~al.(1994)Blum, Furst, Jackson, Kearns, Mansour, and
  Rudich]{BFJKMR94}
Blum, A., Furst, M., Jackson, J., Kearns, M., Mansour, Y., and Rudich, S.
\newblock Weakly learning {DNF} and characterizing statistical query learning
  using {F}ourier analysis.
\newblock In \emph{Proceedings of the 26th Annual ACM Symposium on Theory of
  Computing (STOC)}, pp.\  253--262, 1994.

\bibitem[Breiman(2001)]{Bre01}
Breiman, L.
\newblock Random forests.
\newblock \emph{Machine learning}, 45\penalty0 (1):\penalty0 5--32, 2001.

\bibitem[Breiman et~al.(1984)Breiman, Friedman, Stone, and Olshen]{BFSO84}
Breiman, L., Friedman, J., Stone, C., and Olshen, R.
\newblock \emph{Classification and regression trees}.
\newblock Wadsworth International Group, 1984.

\bibitem[Brutzkus et~al.(2019)Brutzkus, Daniely, and Malach]{BDM19}
Brutzkus, A., Daniely, A., and Malach, E.
\newblock {On the Optimality of Trees Generated by ID3}.
\newblock \emph{ArXiv}, abs/1907.05444, 2019.

\bibitem[Brutzkus et~al.(2020)Brutzkus, Daniely, and Malach]{BDM20}
Brutzkus, A., Daniely, A., and Malach, E.
\newblock {ID3} learns juntas for smoothed product distributions.
\newblock In \emph{Proceedings of the 33rd Annual Conference on Learning Theory
  (COLT)}, pp.\  902--915, 2020.

\bibitem[Bshouty(1993)]{Bsh93}
Bshouty, N.
\newblock Exact learning via the monotone theory.
\newblock In \emph{Proceedings of 34th Annual Symposium on Foundations of
  Computer Science (FOCS)}, pp.\  302--311, 1993.

\bibitem[Bshouty et~al.(2002)Bshouty, Eiron, and Kushilevitz]{BEK02}
Bshouty, N.~H., Eiron, N., and Kushilevitz, E.
\newblock {PAC} learning with nasty noise.
\newblock \emph{Theoretical Computer Science}, 288\penalty0 (2):\penalty0
  255--275, 2002.

\bibitem[Chen \& Moitra(2019)Chen and Moitra]{CM19}
Chen, S. and Moitra, A.
\newblock Beyond the low-degree algorithm: mixtures of subcubes and their
  applications.
\newblock In \emph{Proceedings of the 51st Annual ACM Symposium on Theory of
  Computing (STOC)}, pp.\  869--880, 2019.

\bibitem[Chen \& Guestrin(2016)Chen and Guestrin]{CG16}
Chen, T. and Guestrin, C.
\newblock Xgboost: A scalable tree boosting system.
\newblock In \emph{Proceedings of the 22nd ACM SIGKDD International Conference
  on Knowledge Discovery and Data Mining (KDD)}, pp.\  785--794, 2016.

\bibitem[Dachman-Soled et~al.(2015)Dachman-Soled, Feldman, Tan, Wan, and
  Wimmer]{DSFTWW15}
Dachman-Soled, D., Feldman, V., Tan, L.-Y., Wan, A., and Wimmer, K.
\newblock Approximate resilience, monotonicity, and the complexity of agnostic
  learning.
\newblock In \emph{Proceedings of the 26th Annual Symposium on Discrete
  Algorithms (SODA)}, pp.\  498--511, 2015.

\bibitem[Dietterich et~al.(1996)Dietterich, Kearns, and Mansour]{DKM96}
Dietterich, T., Kearns, M., and Mansour, Y.
\newblock {Applying the weak learning framework to understand and improve
  C4.5}.
\newblock In \emph{Proceedings of the 13th International Conference on Machine
  Learning (ICML)}, pp.\  96--104, 1996.

\bibitem[Fiat \& Pechyony(2004)Fiat and Pechyony]{FP04}
Fiat, A. and Pechyony, D.
\newblock Decision trees: More theoretical justification for practical
  algorithms.
\newblock In \emph{Proceedings of the 15th International Conference on
  Algorithmic Learning Theory (ALT)}, pp.\  156--170, 2004.

\bibitem[Freund(1995)]{freund}
Freund, Y.
\newblock Boosting a weak learning algorithm by majority.
\newblock \emph{Information and computation}, 121\penalty0 (2):\penalty0
  256--285, 1995.

\bibitem[Gopalan et~al.(2008)Gopalan, Kalai, and Klivans]{GKK08}
Gopalan, P., Kalai, A., and Klivans, A.
\newblock Agnostically learning decision trees.
\newblock In \emph{Proceedings of the 40th ACM Symposium on Theory of Computing
  (STOC)}, pp.\  527--536, 2008.

\bibitem[Hancock(1993)]{Han93}
Hancock, T.
\newblock {Learning $k$$\mu$ decision trees on the uniform distribution}.
\newblock In \emph{Proceedings of the 6th Annual Conference on Computational
  Learning Theory (COT)}, pp.\  352--360, 1993.

\bibitem[Haussler(1992)]{Hau92}
Haussler, D.
\newblock Decision theoretic generalizations of the pac model for neural net
  and other learning applications.
\newblock \emph{Information and computation}, 100\penalty0 (1):\penalty0
  78--150, 1992.

\bibitem[Hazan et~al.(2018)Hazan, Klivans, and Yuan]{HKY18}
Hazan, E., Klivans, A., and Yuan, Y.
\newblock Hyperparameter optimization: A spectral approach.
\newblock \emph{Proceedings of the 6th International Conference on Learning
  Representations (ICLR)}, 2018.

\bibitem[Jackson \& Servedio(2006)Jackson and Servedio]{JS06}
Jackson, J.~C. and Servedio, R.~A.
\newblock On learning random dnf formulas under the uniform distribution.
\newblock \emph{Theory of Computing}, 2\penalty0 (8):\penalty0 147--172, 2006.
\newblock \doi{10.4086/toc.2006.v002a008}.
\newblock URL \url{http://www.theoryofcomputing.org/articles/v002a008}.

\bibitem[Kahn et~al.(1988)Kahn, Kalai, and Linial]{KKL88}
Kahn, J., Kalai, G., and Linial, N.
\newblock The influence of variables on boolean functions.
\newblock In \emph{Proceedings of the 29th Annual Symposium on Foundations of
  Computer Science (FOCS)}, pp.\  68--80, 1988.

\bibitem[Kalai(2004)]{Kal04}
Kalai, A.
\newblock Learning monotonic linear functions.
\newblock In \emph{Proceedings of the 17th Annual International Conference on
  Computational Learning Theory (COLT)}, pp.\  487--501. Springer, 2004.

\bibitem[Kalai et~al.(2008{\natexlab{a}})Kalai, Klivans, Mansour, and
  Servedio]{KKMS08}
Kalai, A., Klivans, A., Mansour, Y., and Servedio, R.~A.
\newblock Agnostically learning halfspaces.
\newblock \emph{SIAM Journal on Computing}, 37\penalty0 (6):\penalty0
  1777--1805, 2008{\natexlab{a}}.

\bibitem[Kalai et~al.(2009)Kalai, Samorodnitsky, and Teng]{KST09}
Kalai, A., Samorodnitsky, A., and Teng, S.-H.
\newblock Learning and smoothed analysis.
\newblock In \emph{Proceedings of the 50th Annual IEEE Symposium on Foundations
  of Computer Science (FOCS)}, pp.\  395--404, 2009.

\bibitem[Kalai \& Servedio(2005)Kalai and Servedio]{KS05}
Kalai, A.~T. and Servedio, R.~A.
\newblock Boosting in the presence of noise.
\newblock \emph{Journal of Computer and System Sciences}, 71\penalty0
  (3):\penalty0 266--290, 2005.

\bibitem[Kalai et~al.(2008{\natexlab{b}})Kalai, Mansour, and Verbin]{KMV08}
Kalai, A.~T., Mansour, Y., and Verbin, E.
\newblock On agnostic boosting and parity learning.
\newblock In \emph{Proceedings of the 40th Annual ACM Symposium on Theory of
  Computing (STOC)}, pp.\  629--638, 2008{\natexlab{b}}.

\bibitem[Kearns(1996)]{Kea96}
Kearns, M.
\newblock Boosting theory towards practice: recent developments in decision
  tree induction and the weak learning framework (invited talk).
\newblock In \emph{Proceedings of the 13th National Conference on Artificial
  intelligence (AAAI)}, pp.\  1337--1339, 1996.

\bibitem[Kearns \& Mansour(1996)Kearns and Mansour]{KM96}
Kearns, M. and Mansour, Y.
\newblock On the boosting ability of top-down decision tree learning
  algorithms.
\newblock In \emph{Proceedings of the 28th Annual Symposium on the Theory of
  Computing (STOC)}, pp.\  459--468, 1996.

\bibitem[Kearns et~al.(1994)Kearns, Schapire, and Sellie]{KSS94}
Kearns, M., Schapire, R., and Sellie, L.
\newblock {Toward efficient agnostic learning}.
\newblock \emph{Machine Learning}, 17\penalty0 (2/3):\penalty0 115--141, 1994.

\bibitem[Kushilevitz \& Mansour(1993)Kushilevitz and Mansour]{KM93}
Kushilevitz, E. and Mansour, Y.
\newblock Learning decision trees using the fourier spectrum.
\newblock \emph{SIAM Journal on Computing}, 22\penalty0 (6):\penalty0
  1331--1348, December 1993.

\bibitem[Lee(2009)]{Lee09}
Lee, H.
\newblock \emph{On the learnability of monotone functions}.
\newblock PhD thesis, Columbia University, 2009.

\bibitem[Linial et~al.(1993)Linial, Mansour, and Nisan]{LMN93}
Linial, N., Mansour, Y., and Nisan, N.
\newblock Constant depth circuits, {F}ourier transform and learnability.
\newblock \emph{Journal of the ACM}, 40\penalty0 (3):\penalty0 607--620, 1993.

\bibitem[Long \& Servedio(2009)Long and Servedio]{LS08}
Long, P. and Servedio, R.
\newblock Adaptive martingale boosting.
\newblock In Koller, D., Schuurmans, D., Bengio, Y., and Bottou, L. (eds.),
  \emph{Proceedings of the 22nd Annucal Conference on Advances in Neural
  Information Processing Systems (NeurIPS)}, volume~21. Curran Associates,
  Inc., 2009.
\newblock URL
  \url{https://proceedings.neurips.cc/paper/2008/file/38b3eff8baf56627478ec76a704e9b52-Paper.pdf}.

\bibitem[Long \& Servedio(2005)Long and Servedio]{LS05}
Long, P.~M. and Servedio, R.~A.
\newblock Martingale boosting.
\newblock In \emph{Proceedings of the 18th Annual International Conference on
  Computational Learning Theory (COLT)}, pp.\  79--94. Springer, 2005.

\bibitem[Mansour \& McAllester(2002)Mansour and McAllester]{MM02}
Mansour, Y. and McAllester, D.
\newblock Boosting using branching programs.
\newblock \emph{Journal of Computer and System Sciences}, 64\penalty0
  (1):\penalty0 103--112, 2002.

\bibitem[O'Donnell(2014)]{ODBook}
O'Donnell, R.
\newblock \emph{Analysis of Boolean Functions}.
\newblock Cambridge University Press, 2014.

\bibitem[O'Donnell \& Servedio(2007)O'Donnell and Servedio]{OS07}
O'Donnell, R. and Servedio, R.
\newblock {Learning monotone decision trees in polynomial time}.
\newblock \emph{SIAM Journal on Computing}, 37\penalty0 (3):\penalty0 827--844,
  2007.

\bibitem[O'Donnell et~al.(2005)O'Donnell, Saks, Schramm, and Servedio]{OSSS05}
O'Donnell, R., Saks, M., Schramm, O., and Servedio, R.
\newblock Every decision tree has an influential variable.
\newblock In \emph{Proceedings of the 46th Annual IEEE Symposium on Foundations
  of Computer Science (FOCS)}, pp.\  31--39, 2005.

\bibitem[Quinlan(1986)]{Qui86}
Quinlan, R.
\newblock Induction of decision trees.
\newblock \emph{Machine learning}, 1\penalty0 (1):\penalty0 81--106, 1986.

\bibitem[Quinlan(1993)]{Qui93}
Quinlan, R.
\newblock \emph{C4.5: Programs for Machine Learning}.
\newblock Morgan Kaufmann Publishers Inc., San Francisco, CA, USA, 1993.
\newblock ISBN 1558602402.

\end{thebibliography}
\bibliographystyle{icml2022}

\newpage
\appendix
\onecolumn

\section{Other related work}
\label{apx:other-related-work}

\paragraph{Boosting by branching programs.} Kearns and Mansour~\cite{KM96} (see also~\cite{Kea96,DKM96}) were the first to propose the perspective of viewing impurity-based decision tree algorithms as boosting algorithms.  Their analysis assumes the noiseless setting.  Subsequently, departing from~\cite{KM96}'s motivation of analyzing practical decision tree algorithms, Mansour and McAllester~\cite{MM02} initiated a line of work on boosting by {\sl branching programs}, a variant of decision trees where the underlying graph is a DAG rather than a tree.   While~\cite{MM02} assumes the noiseless setting, the followup works~\cite{Kal04,LS05,KS05,LS08} handle various types of random (i.e.~non-adversarial) label noise, and the work of~\cite{KMV08} handles agnostic noise. 

Our work differs from this line of work in two ways: first and foremost, our results apply to impurity-based decision tree algorithms such as ID3, CART, and C4.5---the overarching goal of our work is to analyze and establish noise tolerance properties of these algorithms that are widely used in practice---whereas branching programs are much less commonly used.  Second, we handle the strongest noise model of nasty noise, whereas these results only allow for corruptions of labels and not features.  

\paragraph{Theoretical work on decision tree learning.}  Decision trees are one of the most intensively studied concept classes in learning theory.  The literature on this problem is rich and vast, spanning over three decades, and it continues to grow.  However, in most of these works, the algorithms analyzed not resemble practical impurity-based decision tree algorithms.  Indeed, most of them are {\sl improper} algorithms, in the sense that their hypotheses are not themselves decision trees. Quoting Kearns and Mansour~\cite{KM96}, ``In summary, it seems fair to say that despite
their other successes, the models of computational learning theory have not yet provided significant insight
into the apparent empirical success of programs like C4.5 and CART."

\section{Bounds with Total Variation Distance}\label{apx:moments_tv_dist}

\begin{lemma}[Expectation and TV-distance]\label{lem:expectation_tv_dist}
Let  $\mcE, \wh{\mcE}$ be two distributions over a domain  $\mcV$ with $\dtv(\mcE, \wh{\mcE}) \leq \eta$ and let $f : \mcV \to [0,1]$. Then

\begin{equation}
    \left|\Ex_{\bx \sim \mcE}[f(\bx)] - \Ex_{\bx \sim \wh{\mcE}}[f(\bx)] \right | \leq \eta.
\end{equation}
\end{lemma}

\begin{proof}
 The result follows immediately from the following definition of total variation distance:
 
 \[
 \dtv(\mcE, \wh{\mcE}) = \sup_{T: \mcV \to [0,1]} \paren{\Ex_{\bx \sim \mcE}[T(\bx)] - \Ex_{\bx \sim \wh{\mcE}}[T(\bx)]}. 
 \qedhere
 \]
\end{proof}

\begin{lemma}[Variance and TV-distance]\label{lem:var_tv_dist}
 Let $\mcE, \wh{\mcE}$ be two distributions over a domain $\mcV$ with $\dtv(\mcE, \wh{\mcE}) \leq \eta$ and let $f : \mcV \to [0,1]$. Then 
\begin{equation*}
    \left | \Var_{\mcE}[f(\bx)] - \Var_{\wh{\mcE}}[f(\bx)]  \right | \leq \eta. 
\end{equation*}
\end{lemma}

\begin{proof}
We can write 

\begin{equation}\label{eq:var_exp}
\Var_{\mcE}[f(\bx)] = \Ex_{\substack{\bx \sim \mcE \\ \bx' \sim \mcE}}\left[\frac{(f(\bx) - f(\bx'))^2}{2} \right].
\end{equation}

If $\dtv(\mcE, \wh{\mcE}) \leq \eta$, then it is easy to see that $\dtv(\mcE^2, \wh{\mcE}^2) \leq 2\eta$, where $\mcE^2$ indicates the product distribution of two independent draws from $\mcE$, Moreover, since $(f(\bx) - f(\bx'))^2 \leq 1$, we can apply \Cref{lem:expectation_tv_dist} to \Cref{eq:var_exp} with $\mcE^2$ and $\wh{\mcE}^2$ to get $\left | \Var_{\mcE}[f(\bx)] - \Var_{\wh{\mcE}}[f(\bx)]  \right | \leq \eta.$
\end{proof}

\begin{lemma}[Covariance and TV-distance]\label{lem:covariance_tv_dist}
 Let $\mcE, \wh{\mcE}$ be two distributions over a domain $\mcV$ with $\dtv(\mcE, \wh{\mcE}) \leq \eta$ and let $f,g : \mcV \to [0,1]$ be two functions. Then
\begin{equation}\label{eq:cov_exp}
    \left | \Cov_{\mcE}[f(\bx), g(\bx)] - \Cov_{\wh{\mcE}}[f(\bx), g(\bx)]  \right | \leq  2\eta.
\end{equation}
\end{lemma}

\begin{proof}

We can write

\begin{equation*}
    Cov_{\mcE}[f(\bx), g(\bx)] 
      =  \Ex_{\substack{\bx \sim \mcE \\ \bx' \sim \mcE}}\left[\frac{(f(\bx) - f(\bx'))(g(\bx) - g(\bx'))}{2} \right] \\ 
\end{equation*}

If $\dtv(\mcE, \wh{\mcE}) \leq \eta$, then it is easy to see that $\dtv(\mcE^2, \wh{\mcE}^2) \leq 2\eta$, where $\mcE^2$ indicates the product distribution of two independent draws from $\mcE$, Moreover, since $f, g: \mcV \to \{0,1\}$, we have $((f(\bx) - f(\bx'))(g(\bx) - g(\bx'))) \in [-1, 1]$, so we can apply \Cref{lem:expectation_tv_dist} to \Cref{eq:cov_exp} with $\mcE^2$ and $\wh{\mcE}^2$ to get $ \left | \Cov_{\mcE}[f(\bx), g(\bx)] - \Cov_{\wh{\mcE}}[f(\bx), g(\bx)]  \right | \leq  2\eta$.
\end{proof}

\begin{lemma}
    \label{lem:TV-leaves}
    For any finite set $L$, (possibly infinite) set $X$, and distributions $\mcD, \widehat{\mcD}$ each over the product domain $L \times X$,
    \begin{equation}
        \label{eq:appendix-exp-tv-bound}
        \Ex_{\bell \sim \mcD_\ell}\left[\dtv(\mcD_{x|\bell}, \widehat{\mcD}_{x| \bell})\right] \leq 2 \dtv(\mcD, \widehat{\mcD})
    \end{equation}
    where $\mcD_{\ell}$ is marginal distribution of $\bell$ for $(\bx, \bell) \sim \mcD$ and $\mcD_{x|\ell}$ is the conditional distribution of $\bx$ when $(\bell, \bx) \sim \mcD$ conditioned upon $\bell = \ell$.
\end{lemma}

Intuitively, we will define a distribution $\mcD'$ that ``mixes" $\mcD$ and $\wh{\mcD}$: To sample $(\bell, \bx) \sim \mcD'$, we first draw $\bell \sim \mcD_{\ell}$ and then $\bx \sim \wh{\mcD}_{x|\bell}$. We'll be able to show that the l.h.s. of \Cref{eq:appendix-exp-tv-bound} is equal to $\dtv(\mcD, \mcD')$.

Then, we'll show that $\dtv(\mcD', \wh{\mcD}) = \dtv(\mcD_{\ell}, \wh{\mcD}_{\ell}) \leq \dtv(\mcD, \wh{\mcD})$. Finally, we can bound $\dtv(\mcD, \mcD') \leq \dtv(\mcD, \wh{\mcD}) + \dtv(\mcD', \wh{\mcD})$ using triangle inequality.

There are many (equivalent) definitions of total variation distance. To formalize the above intuition, we will use the formulation that, for two distributions $\mcD, \wh{\mcD}$ over a domain $\Omega$,
\begin{equation}
    \label{eq:TV-def-sup}
    \dtv(\mcD, \wh{\mcD}) \coloneqq \sup_{A \subseteq \Omega}(\mcD(A) - \wh{\mcD}(A))
\end{equation}
where $\mcD(A)$ is equal to $\Pr_{\bx \sim \mcD}[\bx \in A]$.
\begin{proof}[Proof of \Cref{lem:TV-leaves}]
    We compute,
    \begin{align*}
        \Ex_{\bell \sim \mcD_\ell}&\left[\dtv(\mcD_{x|\bell}, \widehat{\mcD}_{x| \bell})\right] \\
        &= \Ex_{\bell \sim \mcD_\ell}\left[\sup_{X_{\bell} \subseteq X}\left(\mcD_{x|\bell}(X_{\bell}) -  \widehat{\mcD}_{x| \bell}(X_{\bell})\right)\right] \tag{\Cref{eq:TV-def-sup}}\\
        &=\sup_{X_{\bell} \subseteq X\text{ for each }\ell \in L}\left(\Ex_{\bell \sim \mcD_\ell}\left[\mcD_{x|\bell}(X_{\bell}) -  \widehat{\mcD}_{x| \bell}(X_{\bell})\right]\right) \tag{$\sup$ and $\Ex$ commute because $L$ is finite} \\
        &= \sup_{A \subseteq L \times X}\left(\Ex_{\bell \sim \mcD_\ell} \left[\mcD_{x|\bell}(A_{\bell}) -  \widehat{\mcD}_{x| \bell}(A_{\bell})\right]\right) \tag{defining $A_{\ell} \coloneqq \{x| (x, \ell) \in A\}$} \\
        &= \sup_{A \subseteq L \times X}\left(\Ex_{\bell \sim \mcD_\ell} \left[\mcD_{x|\bell}(A_{\bell})\right] - \Ex_{\wh{\bell} \sim \wh{\mcD}_\ell} \left[\wh{\mcD}_{x|\wh{\bell}}(A_{\wh{\bell}})\right]  \right) + \sup_{A \subseteq L \times X}\left(\Ex_{\bell \sim \mcD_\ell} \left[\widehat{\mcD}_{x| \bell}(A_{\bell})\right] - \Ex_{\wh{\bell} \sim \wh{\mcD}_\ell} \left[\wh{\mcD}_{x|\wh{\bell}}(A_{\wh{\bell}})\right]  \right) \tag{triangle inequality}
    \end{align*}
    We analyze each term of the above two terms separately. For the first,
    \begin{align*}
         \sup_{A \subseteq L \times X}&\left(\Ex_{\bell \sim \mcD_\ell} \left[\mcD_{x|\bell}(A_{\bell})\right] - \Ex_{\wh{\bell} \sim \wh{\mcD}_\ell} \left[\wh{\mcD}_{x|\wh{\bell}}(A_{\wh{\bell}})\right]  \right) \\
         &= \sup_{A \subseteq L \times X} \left(\mcD(A) - \wh{\mcD}(A)\right) = \dtv(\mcD, \wh{\mcD}).
    \end{align*}
    Next, we'll bound the second term. Using $p_{\mcD}(\ell)$ as shorthand for $\Pr_{\bell \sim \mcD_{\ell}}[\bell = \ell]$,
    \begin{align*}
        \sup_{A \subseteq L \times X}&\left(\Ex_{\bell \sim \mcD_\ell} \left[\widehat{\mcD}_{x| \bell}(A_{\bell})\right] - \Ex_{\wh{\bell} \sim \wh{\mcD}_\ell} \left[\wh{\mcD}_{x|\wh{\bell}}(A_{\wh{\bell}})\right]  \right) \\
        &= \sup_{A \subseteq L \times X}\sum_{\ell \in L} \left(p_{\mcD}(\ell)\cdot \widehat{\mcD}_{x| \ell}(A_{\ell}) - p_{\wh{\mcD}}(\ell)\cdot \widehat{\mcD}_{x| \ell}(A_{\ell})\right)
    \end{align*}
    The above is maximized by setting $A_{\ell} = X$ whenever $p_{\mcD}(\ell) \geq p_{\wh{\mcD}}(\ell)$ and $A_{\ell} = \emptyset$ otherwise. Therefore,
    \begin{align*}
        \sup_{A \subseteq L \times X}&\left(\Ex_{\bell \sim \mcD_\ell} \left[\widehat{\mcD}_{x| \bell}(A_{\bell})\right] - \Ex_{\wh{\bell} \sim \wh{\mcD}_\ell} \left[\wh{\mcD}_{x|\wh{\bell}}(A_{\wh{\bell}})\right]  \right) \\
        &= \max\left(p_{\mcD}(\ell)- p_{\wh{\mcD}}(\ell), 0\right) \\
        &= \sup_{A' \subseteq L}\left(\mcD_{\ell}(A') - \wh{\mcD}_{\ell}(A')\right) = \dtv(\mcD_{\ell}, \wh{\mcD}_{\ell}).
    \end{align*}
    Finally, we note that $\dtv(\mcD_{\ell}, \wh{\mcD}_{\ell}) \leq \dtv(\mcD, \wh{\mcD})$ as the TV distance of marginal distributions is at most the TV distance of the original distributions. Combining all of the above,
    \begin{equation*}
        \Ex_{\bell \sim \mcD_\ell} \left[\dtv(\mcD_{x|\bell}, \widehat{\mcD}_{x| \bell})\right] \leq \dtv(\mcD, \wh{\mcD}) +  \dtv(\mcD, \wh{\mcD}) = 2 \dtv(\mcD, \wh{\mcD}).  \qedhere
    \end{equation*}

\end{proof}



\section{General Impurity Functions}
\begin{remark}[Other impurity functions]
    \label{remark:other-impurity}
    \Cref{lem:delta_impurity} is the only part of the proof of \Cref{thm:main_formal} that depends on the particular impurity function $\mcG$. That Lemma goes through for any impurity function that, for any constant $\kappa$, satisfies $\mcG''(x) \leq -\kappa$ for all $x \in (0,1)$, though the $16$ in \Cref{eq:delta-drop-16} is replaced with $2\kappa$. This is because (using the fact that $\mu(\mcD_\ell) = \Pr_{\mcD_\ell}[h(\bx) = 0] \cdot \mu(\mcD_{\ell_0}) + \Pr_{\mcD_\ell}[h(\bx) = 1] \cdot \mu(\mcD_{\ell_1}))$, we can bound the local drop in $\mcG$ as
    \begin{equation*}
        \Delta_{\mcD_\ell}(h) \geq \frac{\kappa}{2} \cdot \left(\Pr_{\mcD_\ell}[h(\bx) = 0] \cdot ( \mu(\mcD_{\ell_0}) - \mu(\mcD_\ell))^2) + \Pr_{\mcD_\ell}[h(\bx) = 1] \cdot ( \mu(\mcD_{\ell_1}) - \mu(\mcD_\ell))^2)\right).
    \end{equation*}
    The above holds with equality when $\mcG'(x) = 4x(1-x)$ for $\kappa = 8$. Therefore, the local drop in $\mcG$ will always be at least $\frac{\kappa}{8}$ as large as it is for $\mcG'$. The remainder of the proof of \Cref{thm:main_formal} goes through unmodified except for slight changes to the constants hidden by $O(\cdot)$.
\end{remark}

\section{Learning with samples versus exact expectations}\label{apx:exact_vs_sample}

In the pseudocode provided in  \Cref{fig:pseudocode}, we assume that we can exactly compute the drop in impurity and therefore can exactly compute expectations of the form $\mu(\mcD_{\ell}) = \Ex_{\bx, \by \sim \mcD_{\ell}}[\by]$. If instead, those expectations are estimated using random samples (as in the statements of \Cref{thm:main,thm:monotone-intro} in the introduction), the algorithm only can estimate them to within some tolerance $\pm \tau$. It is straightforward to carry out the proofs of \Cref{thm:main_formal,thm:monotone_dt} accounting for this $\pm \tau$ error as long as $\tau \leq O(\frac{\gamma^2 \eps^2}{t})$.

As the $\BuildTD$ only needs to compute at most $O(t^2 |\mcH|)$ expectations, standard concentration inequalities can be used to show that a sample of size $\tilde{O}(1/\tau^2 \cdot t^2 \cdot |\mcH|)$ is sufficient to compute all expectations to accuracy $\tau$ with high probability. However, the situation with noise is slightly more nuanced, as the adversary gets to see the sample $\bS \sim \mcD^n$ before deciding on $\eta$-nasty-noise corruption $\hat{\bS}$. This means the adversary can choose how to modify empirical estimates after seeing the sample $\bS$. Luckily, Theorem 5 of \cite{BLMT21} handles exactly this case.  It says that as long as the sample has size $\tilde{O}(1/\tau^2 \cdot t^2 \cdot |\mcH|)$, with high probability, all empirical estimates computed on the corrupted sample $\hat{\bS}$ will be within $\pm \tau$ of the true expectations of some $\hat{\mcD}$ that is $\eta$-close to $\mcD$. Therefore, by proving that $\BuildTD$ succeeds on \emph{every} $\hat{\mcD}$ that is $\eta$-close to $\mcD$, as we do in \Cref{thm:main_formal}, we ensure that our algorithm succeeds even if an adversary gets to modify $\eta$-fraction of a sample after seeing it.

\paragraph{Runtime analysis.} Before proving that $\BuildTD$ build low error trees, we will briefly prove that it is efficient. Let $\zeta$ be the time it takes to compute $\mu(\mcD_\ell)$ and $\Pr_{\mcD_{\ell}}[h(x) = 1]$ for some leaf $\ell$ of the tree. Typically, this will be proportional to the size of the data set. Then, as the number of leaves in each iteration will be at most $t$, the time required for an iteration of $\BuildTD$ is at most $O(\zeta t \cdot |\mcH|)$. The algorithm runs for $t$ iterations, so the total runtime is $O(\zeta t^2 \cdot |\mcH|)$.

\section{Lower bounds}
\label{apx:lb}

In this section, we prove \Cref{prop:biased-f}, restated for convenience. 
\begin{proposition}
    \label{prop:biased-f-appendix}
    For any $\eps \in (0, 1/3]$ and $d \geq \log_2(1/\eps)$, for some integer $k \leq d$, there is a monotone function $f: \bits^k \to \bits$ where, for $\bx \sim \bits^k$ chosen uniformly,
    \begin{equation*}
        \min_{b \in \bits}\Pr[f(\bx) = b] \geq \eps
    \end{equation*}
    and,
    \begin{equation}
        \label{eq:correlations-high}
        \Ex[\bx_1 f(\bx)] = \cdots = \Ex[\bx_k f(\bx)] = O\left(\eps \log(1/\eps) \cdot \frac{\log d}{d}\right)
    \end{equation}
\end{proposition}

The function $f$ will be based on the $\Tribes$ function.

\begin{definition}[$\Tribes$]
    \label{def:Tribes} For any $s,w \in \N$, the function $\Tribes_{w,s} : \bits^{ws} \to \bits$ is defined to be the function computed by the read-once disjunctive normal form with $s$ terms (over disjoint sets of variables) of width exactly $w$: 
    \[ \Tribes_{w,s}(x) = (x_{1,1}\wedge \cdots \wedge x_{1,w}) \vee \cdots \vee (x_{s,1} \wedge \cdots \wedge x_{s,w}) \]
    and where we adopt the convention that $-1$ represents logical \textsc{False} and $1$ represents logical \textsc{True}.
\end{definition}

We'll use the following easy to verify facts (see Chapter \S4.2 of~\cite{ODBook}) about $\Tribes$.

\begin{fact}[Properties of \Tribes]
    \label{fact:tribes-properties}
    For any $s, w \in \N$ and $\bx$ uniform in $\bits^{ws}$, 
    \[\Pr[\Tribes_{w,s}(\bx) = -1] = (1 - 2^{-w})^s,\]
    and, for each $i \in [sw]$,
    \[\Ex[x_i\cdot\Tribes_{w,s}(\bx)] = \frac{1}{2^w - 1} \cdot \Pr[\Tribes_{w,s}(\bx) = -1].\]
\end{fact}

\begin{proof}[Proof of \Cref{prop:biased-f-appendix}]
    For any $w \in \N$, let $s_w$ be the largest integer $s$ such that $(1 - 2^{-w})^s \geq \eps$, and let $w^\star$ be the largest integer for which $w s_w \leq d$. As $d \geq \log(1/\eps)$, $w^\star \geq 1$. We will prove that $\Tribes_{w^\star, s_{w^\star}}$ meets all of the criteria of \Cref{prop:biased-f}. Before doing so, we will need to bound $s_w$. Using the Taylor approximation of $\log(1-x)$,
    \begin{equation*}
        (1-2^{-w})^s = \exp(-s(2^{-w} + o(2^{-w}))).
    \end{equation*}
    As a result, we have that
    \begin{equation*}
        s_w = \floor*{\frac{\ln(1/\eps)}{2^{-w} + o(2^{-w})}} = \ln(1/\eps)2^w \cdot (1 + o_{w}(1)).
    \end{equation*}
    Let $k_w = w s_w$. Then $k_{w+1} = k_w \cdot (2 + o_w(1))$. Therefore, the value of $k= w^\star s_{w^\star}$ selected satisfies $k = \Theta(d)$.
    
    By \Cref{fact:tribes-properties}, $\Pr[f(\bx) = -1] \geq \eps$. Furthermore, as $(1-2^{-w^\star})^{s_{w^\star}+1} < \eps$,
    \begin{align*}
        \Pr[f(\bx) = 1] &= (1 - \Pr[f(\bx) = -1])\\
        &= 1 - (1 - 2^{-w^\star})^s_{w^\star} \\
        &= 1 - \frac{(1 - 2^{-w^\star})^{s_{w^\star}+1}}{1 - 2^{-w^\star}} \\
        &> 1 - \frac{\eps}{1/2} \\
        & \geq \frac{1}{3} \geq \eps.
    \end{align*}
    Lastly, we verify \Cref{eq:correlations-high}. As $k = \Theta(d)$,
    \begin{align*}
        \frac{\log d}{d} &= \Theta\left(\frac{\log k}{k}\right) \\
        &= \Theta\left(\frac{\log (\log(1/\eps)\cdot w^\star \cdot 2^{w^\star} )}{\log(1/\eps)\cdot w^\star \cdot 2^{w^\star}}\right) \\
         &\geq \Omega\left(\frac{\log (2^{w^\star} )}{\log(1/\eps)\cdot w^\star \cdot 2^{w^\star}}\right) \\
        &= \Omega\left(\frac{1}{\log(1/\eps) 2^{w^\star}}\right)
    \end{align*}
    Applying the above, \Cref{fact:tribes-properties}, and that $\Pr[f(\bx) = -1] \leq 2\eps$
    \begin{align*}
        \Ex[\bx_1 f(\bx)] = \cdots &= \Ex[\bx_k f(\bx)] = \frac{1}{2^{w^\star} -1} \cdot\Pr[f(\bx) = -1] \\
        &= O\left(\eps \log(1/\eps) \cdot  \frac{\log d}{d}\right). \qedhere
    \end{align*}
\end{proof}

\section{Influence for monotone functions}\label{apx:inf-corr}

\begin{lemma}[Influence = covariance for monotone functions]

Let $\mcD_X = \mcD_X^{(1)} \times \ldots \times \mcD_X^{(d)}$ be an arbitrary product distribution over $\bits^d$. For a monotone function $f : \{\pm 1\}^d \to \bits$ and a feature $i\in [d]$, we have the identity $\Inf_i(f) =  \Cov_{\mcD_X}[f(\bx), \bx_i]$.
\end{lemma}

\begin{proof}
    Let us denote $p_i  = \Pr_{\mcD_X^{(i)}}[\bx_i = 1]$ and $q_i  = 1 - p_i$. We can further define  $\alpha = \Ex_{\bx \sim \mcD_X} \left[ f(\bx) \mid \bx_i = 1\right]$ and $\beta = \Ex_{\bx \sim \mcD_X} \left[ f(\bx) \mid \bx_i = -1\right]$.  We note that $\Ex_{\mcD_X}[f(\bx) \bx_i] = p_i \alpha - q_i\beta$, and $\Ex_{\mcD_X}[ f(\bx)] = p_i\alpha + q_i\beta$, and $\Ex_{\mcD_X}[\bx_i] = p_i - q_i$.
    
    We first expand the definition of covariance:
    \begin{align*}
        \Cov_{\mcD_X}[f(\bx), \bx_i] 
            &= \Ex_{\bx \sim {\mcD_X}}[f(\bx) \bx_i] - \Ex_{\bx \sim {\mcD_X}}[f(\bx)]\Ex_{\bx \sim {\mcD_X}}[\bx_i] \tag{definition of covariance} \\
            &= p_i \alpha - q_i \beta - (p_i \alpha + q_i\beta)(p_i - q_i) \tag{shorthand}\\
            &= \alpha( p_i - p_i(p_i - q_i))  - \beta (q_i  + q_i(p_i - q_i)) \\
            &= \alpha p_i ( 1 - p_i + q_i)  - \beta q_i( 1 - q_i + p_i ) \\
            &= 2p_iq_i(\alpha  - \beta).  \tag{simplification}
    \end{align*}

    We finish by showing that $\Inf_i(f)$ is equal to this last line: 
    \begin{align*}
        \Inf_{i}(f) 
            &= 2 \cdot \Prx_{\substack{\bx \sim {\mcD_X} \\ \bb \sim \mcD_X^{(i)}}} \left[f(\bx) \neq f(\bx_{i=b})\right] \\
            &= \Ex_{\substack{\bx \sim {\mcD_X} \\ \bb \sim \mcD_X^{(i)}}} \left[\left| f(\bx) - f(\bx_{i=b}) \right|\right] \\
            &= \Pr_{\substack{\bx \sim {\mcD_X} \\ \bb \sim {\mcD_X^{(i)}}}}[\bb \ne \bx_i]\cdot \Ex_{\substack{\bx \sim {\mcD_X} \\ \bb \sim {\mcD_X^{(i)}}}} \left[\left| f(\bx) - f(\bx_{i=\bb}) \right|~|~\bb \ne \bx_i\right] \\
%
            %
            &= 2 p_i (1 - p_i)\Ex_{\bx \sim {\mcD_X}} \left[ f(\bx_{i=1}) - f(\bx_{i=-1}) \right] \tag{$f$ is monotone} \\
            &= 2 p_i (1 - p_i) \paren{ \Ex_{\bx \sim {\mcD_X}} \left[ f(\bx) \mid \bx_i = 1\right] -  \Ex_{\bx \sim {\mcD_X}} \left[f(\bx) \mid \bx_{i}=-1 \right] } \\
            &= 2 p_i q_i (\alpha - \beta).  \qedhere
    \end{align*}
\end{proof}

\section{Local drop in $\G$ and covariance}\label{apx:delta_impurity}

\begin{lemma}[Local drop in $\G$ in terms of covariance]
Let $\mcE$ be any distribution over $\mathcal{X}\times \zo$ and  $h : \mcX \to \{0,1\}$ be a splitting function. Then:
\begin{equation}
    \Delta_{\mcE}(h) \geq 16 \cdot \Cov_{\mcE}[h(\bx), \by]^2
\end{equation}
\end{lemma}

\begin{proof}

Let $\tau = \E_\mcE[h(\bx)]$ and $\delta = \E_\mcE[\by ~|~h(\bx) = 1] - \E_\mcE[\by ~|~h(\bx) = 0]$. Equation 20 of~\cite{KM96} states that $\Delta_{\mcE}(h) \ge 4 \tau(1 - \tau)\delta^2$. Then to prove this lemma, it suffices to show that $16 \Cov_{\mcE}[h(\bx), \by]^2 \le 4 \tau(1 - \tau)\delta^2$. We expand the definition of covariance:
\begin{align*}
    &\Cov[h(\bx), \by] = \E[h(\bx)\by] - \E[h(\bx)]\E[\by] \\
    &\quad= \E[\by ~|~ h(\bx) = 1]\E[h(\bx)] - \E[h(\bx)]\E[\by] \\
    &\quad= \E[h(\bx)] (\E[\by ~|~ h(\bx) = 1] - \E[\by])  \\
    &\quad= \E[h(\bx)] (\E[\by ~|~ h(\bx) = 1] - \E[\by~|~h(\bx) = 1]\E[h(\bx)]  \\
     & \qquad   - \E[\by~|~h(\bx) = 0](1 - \E[h(\bx)])) \\
    &\quad= \E[h(\bx)](1 - \E[h(\bx)]) \\
    & \qquad \ (\E[\by ~|~ h(\bx) = 1] - \E[\by~|~h(\bx) = 0]) \\
    &\quad= \tau(1 - \tau)\delta. 
\end{align*}

The lemma follows from the fact that $\tau(1 - \tau) \le 1/4$. 
\end{proof}


\end{document}